\documentclass{article}
\usepackage{geometry}
\usepackage[numbers]{natbib}

\usepackage[textsize=tiny]{todonotes}

\title{Differentially Private Conditional Independence Testing}

\usepackage[shortlabels]{enumitem}
\usepackage{graphicx}
\usepackage{xcolor}
\usepackage{cancel}
\usepackage[normalem]{ulem}
\usepackage{times}
\usepackage{amsmath}
\usepackage{amssymb}
\usepackage{mathtools}
\usepackage{amsthm}
\usepackage{thmtools,thm-restate,thm-autoref}
\usepackage{ulem}

\usepackage{microtype}
\usepackage{graphicx}
\usepackage{wrapfig}
\usepackage{booktabs} 
\usepackage{caption}
\usepackage{subcaption}
\usepackage{setspace}

\usepackage{hyperref}
\hypersetup{colorlinks = true,linkcolor = black, anchorcolor = black, citecolor = black, filecolor = black,urlcolor = black}
\usepackage{algorithm}		
\usepackage[noend]{algpseudocode}

\newcommand{\cov}{\mathrm{cov}}

\newcommand{\Var}{\mathrm{Var}}

\DeclareMathOperator*{\argmax}{arg\,max}

\newcommand{\indep}{\perp \!\!\! \perp}
\newcommand{\notindep}{\cancel{\perp \!\!\! \perp}}

\theoremstyle{plain}
\newtheorem{theorem}{Theorem}[section]

\newtheorem{lemma}[theorem]{Lemma}
\newtheorem{corollary}[theorem]{Corollary}
\newtheorem{claim}[theorem]{Claim}
\theoremstyle{definition}
\newtheorem{definition}[theorem]{Definition}

\theoremstyle{remark}

\newcommand{\Sec}[1]{Section~\hyperref[sec:#1]{\ref*{sec:#1}}} 
\newcommand{\Eqn}[1]{\hyperref[eq:#1]{(\ref*{eq:#1})}} 
\newcommand{\Fig}[1]{\hyperref[fig:#1]{Fig.\,\ref*{fig:#1}}} 
\newcommand{\Tab}[1]{\hyperref[tab:#1]{Tab.\,\ref*{tab:#1}}} 
\newcommand{\Thm}[1]{Theorem\,\hyperref[thm:#1]{\ref*{thm:#1}}} 
\newcommand{\Thms}[2]{Theorems\,\hyperref[thm:#1]{\ref*{thm:#1}} and~\hyperref[thm:#2]{\ref*{thm:#2}}} 
\newcommand{\Fact}[1]{\hyperref[fact:#1]{Fact\,\ref*{fact:#1}}} 
\newcommand{\Lem}[1]{Lemma\,\hyperref[lem:#1]{\ref*{lem:#1}}} 
\newcommand{\Lems}[2]{Lemmas\,\hyperref[lem:#1]{\ref*{lem:#1}} and~\hyperref[lem:#2]{\ref*{lem:#2}}} 
\newcommand{\Prop}[1]{\hyperref[prop:#1]{Prop.\,\ref*{prop:#1}}} 
\newcommand{\Cor}[1]{Corollary~\hyperref[cor:#1]{\ref*{cor:#1}}} 
\newcommand{\Conj}[1]{\hyperref[conj:#1]{Conjecture~\ref*{conj:#1}}} 
\newcommand{\Def}[1]{Definition~\hyperref[def:#1]{\ref*{def:#1}}} 
\newcommand{\Alg}[1]{Algorithm~\hyperref[alg:#1]{\ref*{alg:#1}}} 
\newcommand{\Proc}[1]{\hyperref[proc:#1]{Procedure~\ref*{proc:#1}}} 
\newcommand{\Step}[1]{\hyperref[step:#1]{Step~\ref*{step:#1}}} 
\newcommand{\Steps}[2]{\hyperref[step:#1]{Steps~\ref*{step:#1}} and~\hyperref[step:#2]{\ref*{step:#2}}} 
\newcommand{\Stepss}[3]{\hyperref[step:#1]{Steps~\ref*{step:#1}},~\hyperref[step:#2]{\ref*{step:#2}}, and~\hyperref[step:#3]{\ref*{step:#3}}} 
\newcommand{\Ex}[1]{\hyperref[ex:#1]{Ex.~\ref*{ex:#1}}} 
\newcommand{\Clm}[1]{Claim~\hyperref[clm:#1]{\ref*{clm:#1}}} 
\newcommand{\Inv}[1]{\hyperref[inv:#1]{Invariant~\ref*{inv:#1}}} 
\newcommand{\Rem}[1]{\hyperref[rem:#1]{Remark~\ref*{rem:#1}}} 
\newcommand{\Obs}[1]{\hyperref[obs:#1]{Observation~\ref*{obs:#1}}} 

\newcommand{\Lap}{\mathrm{Lap}}

\newcommand{\R}{\mathbb{R}}
\newcommand{\bR}{\mathbf{R}}
\newcommand{\N}{\mathbb N}

\newcommand{\eps}{\varepsilon}

\newcommand{\cA}{\mathcal{A}}
\newcommand{\cP}{\mathcal{P}}
\newcommand{\cD}{\mathbf{D}}

\newcommand{\cN}{\mathcal{N}}

\newcommand{\bu}{\mathbf{U}}
\newcommand{\bv}{\mathbf{V}}

\newcommand{\bx}{\mathbf{X}}
\newcommand{\by}{\mathbf{Y}}
\newcommand{\bz}{\mathbf{Z}}
\newcommand{\bw}{\mathbf{w}}
\newcommand{\crtconst}{C'}
\newcommand{\gcmconst}{C}
\newcommand{\pto}{\overset{P}{\to}}

\newcommand{\sigmapriv}{\sigma_{\mathrm{priv}}}
\newcommand{\sigmajoint}{\sigma_{\mathrm{joint}}}

\newcommand{\E}{\mathbb{E}}

\author{
    Iden Kalemaj\thanks{This work was conducted during an internship at Amazon.}\\
    Boston Unviersity\\
    \texttt{ikalemaj@bu.edu }
    \and
    Shiva Prasad Kasiviswanathan\\
    Amazon\\
    \texttt{kasivisw@gmail.com}
    \and
    Aaditya Ramdas\\
     Carnegie Mellon University\\
    \texttt{aramdas@cmu.edu}
}
\date{}

\begin{document}

\maketitle

\begin{abstract}
Conditional independence (CI) tests are widely used in statistical data analysis, e.g., they are the building block of many algorithms for causal graph discovery. The goal of a CI test is to accept or reject the null hypothesis that $X \perp \!\!\! \perp Y \mid Z$, where $X \in \mathbb{R}, Y \in \mathbb{R}, Z \in \mathbb{R}^d$. In this work, we investigate conditional independence testing under the constraint of differential privacy. We design two private CI testing procedures: one based on the generalized covariance measure of Shah and Peters (2020) and another based on the conditional randomization test of Cand{\`e}s et al.~(2016) (under the model-X assumption). We provide theoretical guarantees on the performance of our tests and validate them empirically. These are the first private CI tests with rigorous theoretical guarantees that work for the general case when $Z$ is continuous. 
\end{abstract}

\section{Introduction} \label{sec:intro}

Conditional independence (CI) tests are a powerful tool in statistical data analysis, e.g., they are building blocks for graphical models, causal inference, and causal graph discovery~\cite{dawid1979conditional,koller2009probabilistic,pearl2000models}. These analyses are frequently performed on sensitive data, such as clinical datasets and demographic datasets, where concerns for privacy are foremost. For example, in clinical trials, CI tests are used to answer fundamental questions such as “After accounting for (conditioning on) a set of patient covariates  $Z$ (e.g., age or gender), does a treatment $X$ lead to better patient outcomes $Y$?”. 
Formally, given three random variables $(X,Y,Z)$ where $X \in \mathbb{R}$, $Y \in \R$, and $Z \in \R^d$, denote the conditional independence of $X$ and $Y$ given $Z$ by $X \indep Y \mid Z$. Our problem is that of testing
\begin{align*}
    H_0 \text{ (null)}: X \indep Y \mid Z \mbox{ vs. } H_1 \text{ (alternate) }: X \notindep Y \mid Z
\end{align*}
given data drawn i.i.d.\ from a joint distribution of $(X, Y, Z)$. CI testing is a \text{much} harder problem than (unconditional) independence testing, where the variable $Z$ is omitted. Indeed,~\citet{ShahP20} showed that CI testing is a statistically impossible task for  continuous  random variables.\footnote{Any test that uniformly controls the type-I error (false positive rate) for all absolutely continuous triplets $(X, Y, Z)$ such that $X \indep Y \mid Z$, even asymptotically, does not have nontrivial power against \emph{any} alternative.} Thus, techniques for independence testing \emph{do not} extend to the CI testing problem.

When the underlying data is sensitive and confidential, publishing statistics (such as the value of a CI independence test statistic or the corresponding p-value) can leak private information about individuals in the data. For instance, Genome-Wide Association Studies (GWAS) involve finding (causal) relations between Single
Nucleotide Polymorphisms (SNPs) and diseases. CI tests are building blocks for establishing these relations, and the existence of a link between a specific SNP
and a rare disease
may indicate the presence of a minority patient. \emph{Differential privacy}~\cite{DworkMNS16} is a widely studied and deployed 
formal privacy guarantee for data analysis. The output distributions of a differentially private algorithm must look nearly indistinguishable for any two input datasets that differ only in the data of a single individual. In this work, we design the first differentially private (DP) CI tests that can handle continuous variables.

\paragraph{Our Contributions.} We design two private CI tests, each based on a different set of assumptions about the data-generating distribution. {We provide the {\em first} private CI test with rigorous type-I error and power guarantees. Given the aforementioned impossibility results for non-private CI testing, to obtain a CI test with meaningful theoretical guarantees, some assumptions are necessary; in particular we must restrict the space of possible null distributions. In designing our private tests, we start with non-private CI tests that provide rigorous guarantees on type-I error control.

Our first test (Section~\ref{sec:gcm}) is a private version of the generalized covariance measure (GCM) by~\citet{ShahP20}. The type-I error guarantees of the GCM rely on the fairly weak assumption that the conditional means $\E[X \mid Z]$ and $\E[Y \mid Z]$ can be estimated sufficiently well given the dataset size. The
test statistic of the GCM is a normalized sum of the product of the residuals of (nonlinearly)
regressing $X$ on $Z$ and $Y$ on $Z$. This test statistic has \emph{unbounded} sensitivity, thus  a more careful way of adding and analyzing the impact of the privacy noise is needed. 
Our private GCM test adds appropriately scaled, zero-mean noise to the residual products, and calculates the same statistic on the noisy residual products. We show that even with the added noise, the GCM score converges asymptotically to a standard Gaussian distribution under the null hypothesis. 
The magnitude of the noise added to the residuals is constant (it does not vanish with increasing sample size $n$), thus apriori it is surprising that convergence results can be obtained in the presence of such noise. We do so by bounding how the added noise interacts with the noise from the estimation of the residuals.
While our asymptotics justify the threshold for rejecting the null, our private GCM test controls type-I error very well at small finite $n$, as we demonstrate empirically (because the central limit theorem kicks in rather quickly). Recall that finite sample guarantees on type-I error and power are impossible even for non-private CI testing~\cite{ShahP20}. 

The private GCM is the \emph{first} private CI test with rigorous power guarantees. It achieves the same power as the non-private GCM test with a $O(1 /\eps^2)$-inflation of the dataset size.
In addition, while private extensions of non-private hypothesis tests often suffer from introducing excessive type-I error (see the Related Work), the private GCM exhibits the opposite behavior: it can maintain type-I error control even when the non-private GCM test fails to do so. This occurs in scenarios where the regression methods used to estimate the conditional means either underfit or overfit.

 Our second test (Section~\ref{sec:crt}) relies on the \emph{model-X assumption} that the conditional distribution of $X \mid Z$ is known or can be well-approximated. Recently introduced by~\citet{CandsFJL16}, this assumption is useful in settings where one has access to abundant unlabeled data, such as in GWAS, but labeled data are scarce. The model-X assumption is also satisfied in experimental settings where a randomization mechanism is known or designed by the experimenter. CI tests utilizing this assumption provide exact, non-asymptotic, type-I error control~\citep{CandsFJL16, Berrett2019TheCP}, thus bypassing the hardness result of~\citet{ShahP20}. While this assumption has spurred a lot of recent research in (non-private) CI testing, there are no prior private tests in the literature that are designed to work under this assumption. 
 In this work, we focus on the conditional randomization test (CRT)~\citep{CandsFJL16}. We design a private CRT and provide theoretical guarantees on the accuracy of its p-value. 
 We adopt a popular framework for obtaining DP algorithms, known as Report Noisy Max (or the exponential mechanism), which requires defining a problem-specific score function of low sensitivity. To obtain good utility, our score function exploits the specific distribution of intermediate statistics calculated by the CRT.  This score function is novel and can be used for solving a more general problem: given a set of queries on a dataset, estimate privately the rank of a particular query amongst the rest of the queries. 


 We present a detailed empirical evaluation of the proposed tests, justifying their practicality across a wide range of settings. Our experiments confirm that our private CI tests provide the critical type-I error control, and can in fact do so more reliably than their non-private counterparts. As expected, our private tests achieve lower power due to the noise injected for privacy, which can be compensated for with a larger dataset size.

\subsection{Related Work} 

\paragraph{Private Conditional Independence Testing.} \citet{WangPS20} is the only work, prior to ours, to explicitly study private CI testing, motivated by an application to causal discovery. Their tests (obtained from Kendall's $\tau$ and Spearman's $\rho$ score) are designed for categorical $Z$. While these tests could be adapted to work for continuous $Z$ via clustering, in practice this method does not seem to control type-I error, as we show in \Fig{kendall}. The problem worsens with higher-dimensional $Z$.
Our techniques also differ from those of \citet{WangPS20}, who obtain their tests by bounding the sensitivity of non-private CI scores and adding appropriately scaled noise to the true value of the score.
They state two open problems: obtaining private CI tests for continuous $Z$ and obtaining private tests from scores of unbounded sensitivity (as is the case with the GCM score).
We solve both open problems, and manage to privatize the GCM score by instead adding noise to an intermediate statistic, the residuals of fitting $X$ to $Z$ and $Y$ to $Z$.

Another line of work \cite{smith2011privacy, KazanSGB23, PenaB23} has utilized the ``subsample and aggregate'' framework of differential privacy~\cite{NissimRS07} to obtain private versions of existing hypothesis tests in a black-box fashion. 
In this approach, the dataset is partitioned into $k$ smaller datasets; the non-private hypothesis test is evaluated on the smaller datasets; and finally, the results are privately aggregated. Based on this method, \citet{KazanSGB23} propose a test-of-tests (ToT) framework to construct a private version of any known (non-private) hypothesis test. However, they show guarantees on the power of their test based on finite-sample guarantees of the power of the non-private hypothesis test. Since finite-sample guarantees are impossible for CI testing, their method gives \emph{no} power guarantees for CI testing, and thus cannot be reliably used in practice. In addition, in \Fig{kendall} we compare the type-I error control of our tests with the ToT framework and show that it can fail to control type-I error.

\citet{smith2011privacy} analyzed the asymptotic properties of subsample-and-aggregate and showed that for a large family of statistics, one can obtain a corresponding DP statistic with the same asymptotic distribution as the original statistic.
In particular, the result of \citet{smith2011privacy} can be applied to obtain a DP version of the GCM statistic. However, compared to our results on the private GCM, (a) only a weaker notion of privacy, known as \emph{approximate} DP,  would be guaranteed, and (b) an additional condition on the data-generating distribution would have to be introduced, to guarantee a bounded third moment of the GCM statistic. 
Finally, the test of~\citet{PenaB23} only outputs a binary accept/reject decision and not a p-value as our tests provide, and was empirically outperformed by the test of~\citet{KazanSGB23}.


\paragraph{Private (non-conditional) Independence Testing.} A line of work on private independence testing has focused on privatizing the chi-squared statistic \cite{VuS09, JohnsonS13, UhlerSF13, YuFSU14, WangLK15, GaboardiLRV16, RogersK17}. These tests operate with categorical $X$ and $Y$. 
Earlier works obtained private hypothesis tests by adding noise to the histogram of the data \cite{JohnsonS13}, but it was later pointed out that this approach does not provide reliable type-I error control at small sample sizes \cite{FienbergRY10}. Consequent works used numerical approaches to obtain the distribution of the noisy statistic and calculate p-values with that distribution \cite{UhlerSF13,YuFSU14, WangLK15, GaboardiLRV16}, whereas~\citet{RogersK17} obtain new statistics for chi-squared tests whose distribution after the privacy noise can be derived analytically. In this light, one important feature of our private GCM test is that its type-I error control can be more reliable than for the non-private GCM, even at small $n$, as our experiments demonstrate. 
For continuous $X$ and $Y$, \citet{KusnerSSW16} obtained DP versions of several dependence scores (Kendall's $\tau$, Spearman's $\rho$, HSIC), however, they do not provide type-I error or power guarantees. In follow-up work,~\citet{KimS23} obtained private versions of permutation tests, that were applied to kernel-based independence tests, such as HSIC.
Note that CI testing is a much harder task than independence testing, and techniques for the latter do not necessarily translate to CI testing. Our work is part of the broader literature on private hypothesis testing~\cite{BarrientosRMC17, CampbellBRG18, SwanbergGGRGB19, CouchKSBG19, WangKLK18, AwanS18, BrennerN14, Ding21, Vepakomma22}.

\paragraph{Non-private Conditional Independence Testing.}
 A popular category of CI tests are kernel-based tests, obtained by extending the Hilbert-Schmidt independence criterion to the conditional setting ~\cite{fukumizu2007kernel, ZhangPJS11, strobl2019approximate}. However, these tests only provide a weaker \emph{pointwise} asymptotic validity guarantee. It is widely acknowledged that for a statistical test to be useful in practice, it needs to provide the stronger guarantees of either valid level at finite sample size or \emph{uniformly} asymptotic level. Our private GCM test provides the latter guarantee.

One way of getting around the hardness result of~\citet{ShahP20} is through the model-X assumption, where the conditional distribution of $X \mid Z$ is assumed to be accessible.
Tests based on this assumption, such as CRT (conditional randomization test)~\citep{CandsFJL16} and CPT (conditional permutation test)~\citep{Berrett2019TheCP}, provide a general framework for conditional independence testing, where one can use their test statistic of choice and exactly (non-asymptotically) control the type-I error regardless of the data dimensionality.

\section{Preliminaries} 

In this section, we introduce notation used in the paper as well as relevant background. In \Sec{dp}, we introduce background on differential privacy. In \Sec{stats} we provide background on hypothesis testing (including standard definitions of p-value, type-I error, uniform asymptotic level, power, etc.). Finally, in \Sec{krr} we state a result of~\citet{KusnerSSW16} used in our paper on the residuals of kernel ridge regression 


\paragraph{Notation.} If $(V_{P,n})_{n\in \N,P \in \cP}$ is a family of
sequences of random variables whose distributions are determined by $P 
\in \cP$, we say $V_{P,n} =
o_{\cP}(1)$ if for all $\delta > 0$, $\sup_{P \in \cP} \Pr_{P}[|V_{P, n}| > \delta] \to 0$. Similarly,   $V_{P,n} = O_{\cP}(1)$ if for all $\delta > 0$, $\exists M > 0$ such that $\sup_{n \in \N}\sup_{P \in \cP} \Pr_{P}[|V_{P, n}| > M] < \delta$. 

\subsection{Background on Differential Privacy}\label{sec:dp}

The notion of neighboring datasets is central to differential privacy. In this work, we consider datasets $\cD = (\bx, \by, \bz)$ of $n$ datapoints $\{(x_i,y_i,z_i)\}_{i=1}^n$, drawn i.i.d.\ from a joint distribution $P$ on some domain $\mathcal{X} \times \mathcal{Y} \times \mathcal{Z}$. Let $\mathcal{D}$ denote the universe of datasets. A dataset $\cD' \in \mathcal{D}$ is a neighbor of $\cD$ if it can be obtained from $\cD$ by replacing at most one datapoint $(x_i, y_i, z_i) \in \cD$ with an arbitrary entry $(x'_i,y'_i,z'_i) \in \mathcal{X} \times \mathcal{Y} \times \mathcal{Z}$, for some $i \in [n]$. For the purposes of CRT, where we use the distributional information about $X \mid Z$ to resample additional data, we define  $\cD$ to include the new samples (see Section~\ref{sec:crt}).

\begin{definition}[Differential privacy~\cite{DworkMNS16}]\label{def:dp}
A randomized algorithm $\text{Alg}$ is $\eps$-DP if for all neighboring datasets $\cD,\cD'$ and all events $\mathcal{R}$ in the output space of $\text{Alg}$, it holds
$\Pr[\text{Alg}(\cD) \in \mathcal{R}] \leq \exp(\eps) \cdot \Pr[\text{Alg}(\cD') \in \mathcal{R}],$
where the probability is over the randomness of the algorithm. 
\end{definition}   

The Laplace mechanism is a widely used framework for obtaining DP algorithms~\citep{DworkMNS16}.

\begin{definition}[$\ell_1$-sensitivity] \label{def:laplace}
For a function $f \colon \mathcal{D} \to \R^d$, its $\ell_1$-sensitivity $\Delta_f$ is defined as 
$$
    \Delta_f = \max_{\cD, \cD' \mathrm{neighbors}} \lVert f(\cD) - f(\cD') \rVert_1.
$$
\end{definition}
\begin{lemma}[Laplace Mechanism~\citep{DworkMNS16}] 
\label{lem:laplace}
Let $\eps > 0$ and $f \colon \mathcal{D} \to \R^d$ be a function with $\ell_1$-sensitivity $\Delta_f$. Let $W \sim \Lap(0,\Delta_f/\eps)$ be a noise vector from the  Laplace
distribution with scale parameter $\Delta_f/\eps$. The Laplace Mechanism that, on input  $\cD$ and $\eps$, outputs $f(\cD) + W$ is $\eps$-DP.
\end{lemma}

Differential privacy satisfies a post-processing property.

\begin{lemma}[Post-Processing~\cite{DworkMNS16}]\label{lem:postprocess} If the algorithm $\cA$ is $\eps$-differentially private, and $g$ is any randomized function, then the algorithm $g(\cA(x))$ is $\eps$-differentially private. 
\end{lemma}

\subsection{Background on Hypothesis Testing} \label{sec:stats}
Let $\cP$ be the class of the joint distributions for the random variables $(X, Y, Z)$. We say $P \in \cP$ is a null distribution if $X \indep Y \mid Z$. The null-hypothesis, denoted $\cP_0$, is the class of null distributions, 

\paragraph{Type-I error and validity.} Consider a (potentially) randomized test $\varphi_n$ that is run on $n$ samples from a distribution $P \in \cP$ and outputs a binary decision: $1$ for rejecting the null hypothesis and $0$ for accepting the null hypothesis. The quantity $\Pr_P[\varphi_n = 1]$, where $P$ is a null distribution, refers to the type-I error of the test, i.e., the probability that it erroneously rejects the (true) null hypothesis. Given level $\alpha \in (0,1)$ and the null hypothesis $\cP_0$, we say that the test has \emph{valid level at sample size $n$} if the type-I error is bounded by $\alpha$, i.e.,:
\begin{align*}
    \sup_{P \in \cP_0}\Pr_P[\varphi_n = 1] \leq \alpha. 
\end{align*}
The sequence $\{  \varphi_n \}_{n=1}^\infty$ has
\begin{align*}
    \textit{uniformly asymptotic level } \quad &\text{if} \quad \lim_{n \to \infty} \sup_{P \in \cP_0} \Pr_P[\varphi_n = 1] \leq \alpha, \\
    \textit{pointwise asymptotic level } \quad &\text{if} \quad \sup_{P \in \cP_0}  \lim_{n \to \infty} \Pr_P[\varphi_n = 1] \leq \alpha. 
\end{align*}
Usually, we want at least uniformly asymptotic level to hold for a test. Otherwise, for any sample size $n$, there can be some null distribution that does not control type-I error at that sample size. 

A hypothesis test is usually derived from a statistic $T_n$ (such as the GCM statistic) calculated on $n$ samples drawn i.i.d from the distribution $P$ of $(X, Y, Y)$. Having obtained a value $t$ for the statistic $T_n$, the two-sided p-value is:
\begin{align*}
    \text{p-value} = \Pr_{P}[|T_n| \geq t ].
\end{align*}
The hypothesis test $\varphi_n$ with desired validity level $\alpha$ can then be defined as
\begin{align}\label{eq:hypothesis_test}
    \varphi_n = \mathbf{1}[\text{p-value} \leq \alpha].
\end{align}
 Therefore, to obtain a test with the desired validity we need to compute the p-value. 
 The p-value is typically calculated using information about the distribution of $T_n$. We say that $T_n$ converges uniformly over $\cP_0$ to the standard Gaussian distribution if:
\begin{align*}
        \lim_{n \to \infty} \sup_{P \in \cP_0} \sup_{t \in \R} |\Pr_{P}[T_n \leq t] - \Phi(t)| = 0,
\end{align*}
where $\Phi$ is the CDF of the standard Gaussian.
For the GCM statistic, we are given that (under mild assumptions) $T_n$ converges uniformly over the null hypothesis $\cP_0$ to a standard Gaussian distribution \cite{ShahP20}. Thus, if we set $\text{p-value}=  2(1- \Phi(|t|)$ and define the hypothesis test $\varphi_n$ as in \Eqn{hypothesis_test}, we obtain that $\varphi_n$ has uniformly asymptotic level $\alpha$. 

\paragraph{Power.} Once we have a test with uniformly asymptotic level, we would also like the test to correctly accept the alternate hypothesis, when this hypothesis holds. Let $\mathcal{Q}$ be the set of alternate distributions (for which $X \notindep Y \mid Z$). The power of a test is the probability that it correctly rejects the null hypothesis, given that the alternate hypothesis holds. A sequence $\{\varphi_n\}_{n=1}^{\infty}$ of tests has  
\begin{align*}
    \textit{uniformly asymptotic power} \quad \text{if} \quad  \lim_{n \to \infty} \inf_{P \in \mathcal{Q}} \Pr_P[\varphi_n = 1] = 1. 
\end{align*}

\subsection{Residuals of Kernel Ridge Regression}\label{sec:krr}

In our algorithms and experiments, we use kernel ridge regression (KRR) as a procedure for regressing $\bx$ and $\by$ on $\bz$, and rely on the following result by~\citet{KusnerSSW16} about the sensitivity of the residuals of KRR.\footnote{One could also use other regression techniques within our private GCM and private CRT frameworks, and theoretical guarantees continue to hold if similar ($\approx O(1/n)$) bounds on the sensitivity of the residuals are true.}

\begin{theorem}[Restated Theorem 5 of \citet{KusnerSSW16}] \label{thm:kusner}
Let $(\bu,\bv)$ be a dataset of $n$ datapoints 
$(u_i, v_i)$, $i \in [n]$ from the domain $\mathcal{U} \times \mathcal{V} \subset \R \times \R^d$. Suppose that $|\mathcal{U}| \leq 1$. Given a Hilbert space $\mathcal{H}$, let $\bw$ be the vector that minimizes the kernel ridge regression objective 
$$(\lambda/2) \lVert \bw \rVert^2_{\mathcal{H}} + (1/n)\sum_{i=1}^n (u_i - \bw^\top \phi(v_i))^2,$$
for kernel $\phi \colon \R^d \to \mathcal{H}$ with $\lVert \phi(v) \rVert_{\mathcal{H}} \leq 1$ for all $v \in \mathcal{V}$. Define $\bw'$ analogously for a neighboring dataset $(\bu', \bv')$  that is obtained by replacing one datapoint in $(\bu,\bv)$. Then $\lVert \bw \rVert_{\mathcal{H}} \leq \sqrt{2/\lambda}$ and for all $v \in \bv$ it holds: 
$$
|\bw^\top \phi(v) -  \bw'^\top \phi(v)| \leq 8\sqrt{2}/(\lambda^{3/2}n) + 8/(\lambda n).
$$
\end{theorem}

\section{Private Generalized Covariance Measure}\label{sec:gcm}

Here, we present our private Generalized Covariance Measure (GCM) test.
Missing proofs are in Appendix~\ref{app:gcm}.

\paragraph{GCM Test.} We first describe the non-private GCM test of~\citet{ShahP20}. Given a joint distribution $P$ of the random variables $(X,Y,Z)$, we can always write:  
\begin{align*}
    X = f_P(Z) + \chi_P, \quad Y = g_P(Z) + \xi_P,
\end{align*}
where $f_P(z) =\E_P[X|Z=z]$, $g_P(z) = \E_P[Y|Z=z]$, $\chi_P = X -  f_P(z)$, and $\xi_P = Y -  g_P(z)$.

 Let $\cD = (\bx, \by, \bz)$ be a dataset of $n$ i.i.d.\ samples from $P$. Let $\hat{f}$ and $\hat{g}$ be approximations to the conditional expectations $f_P$ and $g_P$, obtained by fitting $\bx$ to $\bz$ and $\by$ to $\bz$, respectively. We consider the products of the residuals from the fitting procedure:
\begin{align}
    R_i = ((x_i - \hat{f}(z_i)) (y_i - \hat{g}(z_i)) \text{ for } i \in [n]. \label{res_products}
\end{align}
The GCM test statistic $T$ is defined as the normalized mean of the residual products, i.e.,
\begin{align}\label{eq:t}{\small
    T(R_1, \dots, R_n) = \frac{ \frac{1}{\sqrt{n}}\sum_{i=1}^n R_i}{ (\frac{1}{n}\sum_{i=1}^n R_i^2 - (\frac{1}{n} \sum_{k=1}^n R_k)^2)^{1/2}  }. 
    }
\end{align}
The normalization critically ensures that $T$ follows a standard normal distribution asymptotically. However, it also leads to the unbounded sensitivity of the statistic $T$.

\paragraph{Private GCM Test.} To construct a DP version of the GCM test, we focus on the vector of residual products, $\bR = (R_1,\dots,R_n)$. Let $\Delta$ denote the $\ell_1$-sensitivity of $\bR$. Given $\Delta$, we use the Laplace mechanism (Lemma~\ref{lem:laplace}) to add scaled Laplace noise to $\bR$ and then compute $T$ on the noisy residual products. The private GCM test we present (in Algorithm~\ref{alg:gcm}) can be used with any fitting procedure, as long as a bound on the sensitivity of the residuals for that procedure is known. 
\vspace*{-1ex}
\begin{algorithm}[h]
  \caption{Private Generalized Covariance Measure} \label{alg:gcm}
  \begin{algorithmic}[1]
    \State {\textbf{Input: }}Dataset $(\bx, \by, \bz) = \{(x_i, y_i, z_i)\}_{i=1}^n$, privacy parameter $\eps > 0$, fitting procedure $\mathcal{F}$, bound $\Delta > 0$ on the sensitivity of residual products of $\mathcal{F}$. 
    \State Let $\hat{f} = \mathcal{F}(\bz, \bx)$ and $\hat{g} = \mathcal{F}(\bz, \by)$  
    \For{ $i = 1, \dots, n$} 
    \State $r_{X, i} \leftarrow x_i - \hat{f}(z_i)$, $r_{Y, i} \leftarrow y_i - \hat{g}(z_i)$ 
    \State $R_i \leftarrow r_{X, i} \cdot r_{Y, i}$ 
    \State $W_i \sim \mathrm{Lap}(0, \Delta/\eps)$ 
    \EndFor
    \State Let $T^{(n)} \leftarrow T(R_1+W_1, \dots, R_n + W_n)$ \Comment{See~\Eqn{t}}
    \State Output p-value $ = 2\cdot(1-\Phi(|T^{(n)}|)$ 
  \end{algorithmic}
\end{algorithm}
\vspace*{-1ex}

While the algorithm uses a simple noise addition strategy to privatize the GCM test, its asymptotic behavior is rather unexpected. Notice that the noise added to the residuals has constant variance that does not vanish with $n\to \infty$. Yet, we show (in \Thm{level_rescaled}) that $T^{(n)}$ converges to a standard Gaussian distribution similarly to the non-private GCM statistic. 
The key step in the analysis is to show that the error introduced by the noise random variables grows at a slower rate than the error introduced by the fitting procedure.
Our algorithmic framework opens up a question of which fitting procedures have residuals with constant bounded sensitivity. We show such a result when Kernel Ridge Regression is used as the fitting method.

Next, we show theoretical guarantees on the type-I error control and power of \Alg{gcm} under mild assumptions on the fitting procedures $\hat{f}$ and $\hat{g}$, listed in Definition~\ref{def:assumption}.

\begin{definition}[Good fit]
\label{def:assumption}
    Consider $u_P(z) = \E_{P}[\chi_P^2 \mid Z = z]$, $v_P(z) =  \E_{P}[\xi_P^2 \mid Z = z]$, and the following error terms:
\begin{align}
    A_f &= \frac{1}{n}\sum_{i=1}^n (f_P (z_i) - \hat{f}(z_i))^2,  
    \quad\quad B_f = \frac{1}{n}\sum_{i=1}^n (f_P (z_i) - \hat{f}(z_i))^2v_P(z_i), \nonumber \\
    A_g &= \frac{1}{n}\sum_{i=1}^n (g_P (z_i) - \hat{g}(z_i))^2,  
    \quad\quad B_g = \frac{1}{n}\sum_{i=1}^n (g_P (z_i) - \hat{g}(z_i))^2u_P(z_i). \label{eq:a_and_b}
\end{align}
Let $\cP$ be a class of distributions for $(X, Y, Z)$ that are continuous with respect to the Lebesgue measure and such that $\sup_{P \in \cP} \E[|\chi_P \xi_P|^{2+\eta}] \leq c$ for some constants $c, \eta > 0$. 
We say the classifiers $\hat{f}$ and $\hat{g}$ are a good fit for $\cP$ 
if $A_f A_g = o_{\cP}(n^{-1}), B_f =  o_{\cP}(1)$, and $B_g = o_{\cP}(1)$.  
\end{definition}

The key part of Definition~\ref{def:assumption} is that the product of mean squared errors of the fitting method (for fitting $X$ to $Z$ and $Y$ to $Z$) converge to zero at a sublinear rate in the sample size. This is known as the ``doubly robust" assumption and is common in many analyses in theoretical ML~\citep{chernozhukov2018double}. It is a mild assumption because it only requires one of two regressions, either fitting $X$ to $Z$ or fitting $Y$ to $Z$, to be sufficiently accurate.
These assumptions are in fact slightly weaker than those of \citet{ShahP20} for guaranteeing uniformly asymptotic level\footnote{Given a level $\alpha \in  (0, 1)$ and null hypothesis $\cP_0$, a test $\psi_n$ has uniformly asymptotic level if its asymptotic type-I error is bounded by $\alpha$ over all distributions in $\cP_0$, i.e., $\lim_{n \to \infty} \sup_{P \in \cP_0} \Pr_P[\psi_n \mbox{ rejects null}] \leq \alpha$. See \Sec{stats}.} and power of the GCM, as we do not require a lower bound on the variance $\E[\chi_P^2 \xi_P^2]$ of the true residuals. This requirement is no longer necessary as we add finite-variance noise to the residual products.

\paragraph{Type-I Error Control.} We show that as with the GCM test of~\citet{ShahP20}, the private counterpart has uniformly asymptotic level. 
While the original GCM test of~\citet{ShahP20} does not require the input variables to be bounded, we assume bounded random variables $X$ and $Y$ to obtain bounds on the sensitivity $\Delta$ of the residual products.  For the rest of this section, we assume publicly known bounds $a$ and $b$  on the domain $\mathcal{X}$ of $X$ and $\mathcal{Y}$ of $Y$, (i.e., $|x| \leq a, \forall x \in \mathcal{X}$ and $|y| \leq b, \forall y \in \mathcal{Y}$).\footnote{
These bounds can also be replaced with high probability bounds, but the privacy guarantees of our CI test would be replaced with what is known as {\em approximate differential privacy}.}
Note that we do not assume such bounds on the domain of $Z$, which is important as $Z$ could be high-dimensional.

\begin{restatable}{theorem}{levelrescaled}{\emph{(Type-I Error Control of Private GCM)}}
\label{thm:level_rescaled}
Let $a$ and $b$ be known bounds on the domains of $X$ and $Y$, respectively. Given a dataset $\cD = (\bx, \by, \bz)$,  let $(\hat{\bx}, \hat{\by}, \bz)$ be the rescaled dataset obtained by setting $\hat{\bx} = \bx/ a$ and $\hat{\by} = \by /b$. Consider $R_i, i \in [n]$, as defined in \eqref{res_products}, for the rescaled dataset $(\hat{\bx}, \hat{\by}, \bz)$. 
Let $W_i \sim \mathrm{Lap}(0, \Delta/\eps)$ for $i \in [n]$, where $\Delta, \eps > 0$ are constants. Let $\cP$ be the set of null distributions from Definition~\ref{def:assumption} for which $\hat{f}, \hat{g}$ are a good fit. 
The statistic $T^{(n)} = T(R_1 + W_1, \dots, R_n + W_n)$, defined in Algorithm~\ref{alg:gcm}, converges uniformly over $\cP$ to the standard Gaussian distribution $\cN(0, 1)$.
\end{restatable}

Since $T^{(n)}$ converges uniformly to a standard Gaussian distribution, this implies that the CI test in Algorithm~\ref{alg:gcm} has uniformly asymptotic level (see Section~\ref{sec:stats}). It also satisfies a weaker pointwise asymptotic level guarantee that holds under slightly weaker assumptions (\Thm{level_rescaled_full}). Note that this result is independent of the bound on $|X|$ and $|Y|$.
While the guarantee of \Thm{level_rescaled} is asymptotic, type-I error control ``kicks in'' at very small sample sizes (like $n = 100$) as confirmed by our experimental results. This behavior is typical for many statistics that converge to a standard Gaussian distribution. Recall that finite-sample guarantees are impossible even for non-private CI testing, even under the assumptions of Definition~\ref{def:assumption}, since these assumptions are only asymptotic in nature.

\paragraph{Noise Addition Leads to Better Type-I Error Control.} A beneficial consequence of the privacy noise is that there are scenarios, under the null hypothesis, where the non-private GCM fails to provide type-I error control, but our private GCM does. If the functions $\hat{f}$ and $\hat{g}$ fail to fit the data (i.e.,~the conditions on $A_f, A_g, B_f, B_g$ in \Thm{level_rescaled} are violated), private GCM can still provide type-I error control. We show in \Sec{experiments} one such scenario, when the learned model underfits the data. Consider on the other hand the case when the model overfits, and more extremely, when the model interpolates asymptotically, i.e.~$\hat{f}(z_i) \to x_i$ and $\hat{g}(z_i) \to y_i$ as $n \to \infty$ for all $i \in [n]$ \cite{LiangR18}. It is not too hard to show that convergence to the standard Gaussian still holds for the private GCM, and thus type-I error control is provided. Instead, the rejection rate of the non-private GCM converges to 1 when the model interpolates.

\paragraph{Power of the Private GCM.} 
  Next, we show a result on the power of our private GCM test. 
  
\begin{restatable}{theorem}{powerrescaled}{\emph{(Power of Private GCM).}}
\label{thm:power_rescaled}
    Consider the setup of \Thm{level_rescaled}. 
    Define the ``signal'' ($\rho_P$) and ``noise'' ($\sigma_P$) of the true residuals $\chi_P, \xi_P$ as: 
    \begin{align*}
        \rho_P = \E_P[\cov(X, Y \mid Z)], \quad  \sigma_P = \sqrt{\mathrm{Var}_P(\chi_P\xi_P)}.
    \end{align*}
   Let $\cP$ be the set of alternate distributions from Definition~\ref{def:assumption} for which $\hat{f}, \hat{g}$ are a good fit. Then
   \begin{align}
       T^{(n)} - \frac{\sqrt{n}\rho_P}{\sigma_P'},\label{eq:power_main} \quad  \mbox{ where }   \sigma_P' =  \sqrt{ \sigma_P^2 + (\frac{\sqrt{2}ab\Delta}{\eps} )^2} 
    \end{align}
    converges uniformly over $\cP$ to $\mathcal{N}(0,1)$. 
\end{restatable}

\paragraph{Discussion on Power.}
   \Thm{power_rescaled} implies that $T^{(n)}$ has uniform (asymptotic) power of $1$ if $\rho_P \neq 0$. See \Cor{power_of_1} for a short proof. In \Thm{power_rescaled_full}, we also show a pointwise (asymptotic) power guarantee, under weaker assumptions.
    We remark that the bounds $a$ and $b$ on $|X|$ and $|Y|$ could depend on the dataset size $n$. \Alg{gcm} has uniform asymptotic power of $1$ as long as $a \cdot b = o(\sqrt{n})$. 

   \citet{ShahP20} show a similar result on the power of the (non-private) GCM, but with $\sigma_P' = \sigma_P$. Suppose $\sigma_P = 1$. Then, \Thm{power_rescaled} states that a $(\frac{ab\Delta}{\eps})^2$-factor   of the dataset size used in the non-private case is required to obtain the same power in the private case. A blow-up in the sample size is typical in DP analyses~\citep{DworkR14}. 

 \paragraph{Private GCM with Kernel Ridge Regression (PrivGCM).}\label{sec:residuals}
 To obtain a bound on the sensitivity of the vector of residual products, we use kernel ridge regression (KRR) as the model for regressing $X$ on $Z$ and $Y$ on $Z$, respectively. Let PrivGCM denote \Alg{gcm} with KRR as the fitting procedure and the bound on $\Delta$.
 
The vector of residual products has $\ell_1$-sensitivity $O_\lambda(1)$ (as formally shown in  \Lem{sensitivity_residuals} using \Thm{kusner}). Along with \Lem{laplace}, this implies that PrivGCM is $\eps$-DP. 
In addition, as shown by \citet{ShahP20}, the requirements on $A_f, A_g, B_f, B_g$ are satisfied when using KRR. If the additional conditions listed in Definition~\ref{def:assumption} are also satisfied, then PrivGCM has uniformly asymptotic level and uniform asymptotic power of $1$ (see \Cor{privgcm}).

\begin{restatable}{lemma}{sensitivityresiduals}{\emph{(Sensitivity of residual products).}}
\label{lem:sensitivity_residuals}
     Let $\mathbf{R}$ be the vector of residual products, as defined in \eqref{res_products}, of fitting a KRR model of $\bx$ to $\bz$ and $\by$ to $\bz$ with regularization parameter $\lambda > 0$. If $|x_i|, |y_i| \leq 1$ for all $i \in [n]$, then $\Delta_{\mathbf{R}} \leq \gcmconst$ where $\gcmconst = 
     4( 1 + \frac{\sqrt{2}}{\sqrt{\lambda}})(1 + \frac{\sqrt{2}}{\sqrt{\lambda}} + \frac{4\sqrt{2}}{\lambda^{3/2}} + \frac{4}{\lambda})$.
\end{restatable}

\section{Private Conditional Randomized Testing} \label{sec:crt}

Here, we propose a private version of the conditional randomization test (CRT), which uses access to the distribution of $X \mid Z$ as a key assumption. Recall that such an assumption is useful, for example, when one has access to abundant unlabeled data $(X,Z)$. Missing proofs are in Appendix~\ref{app:crt}.

\paragraph{CRT.} As before, consider a dataset $(\bx, \by, \bz)$ of $n$ i.i.d.\ samples $(x_i,y_i,z_i), i \in [n]$ from the joint distribution $P$. For ease of notation, denote the original $\bx$ as $\bx^{(0)}$. The key idea of CRT is to sample $m$ copies of $\bx^{(0)}$ from $X \mid Z$, where $Z$ is fixed to the values in $\bz$. That is, for $j \in [m]$ and $i \in [n]$, a new datapoint $x_i^{(j)}$ is sampled from $X \mid Z = z_i$. Then $\bx^{(j)} = (x_1^{(j)}, \dots, x_n^{(j)})$.

 Under  the null hypothesis, the triples $(\bx^{(0)}, \by,\bz), \dots,(\bx^{(m)}, \by,\bz)$ are identically distributed. Thus, for every statistic $T$ chosen independently of the data, the random variables $T(\bx^{(0)}, \by,\bz), \dots, T(\bx^{(m)}, \by,\bz)$ are also  identically distributed. Denote these random variables by $T_0, \dots, T_m$. The p-value is computed by ranking $T_0$, obtained by using the original $\bx^{(0)}$ vector, against $T_1, \dots, T_m$, obtained from the resamples:
 \begin{align*}
     \textrm{p-value} = \frac{1+ \sum_{j=1}^m  \mathbf{1}(T_j \geq T_0)}{m+1}.
 \end{align*}
For every choice of $T$, the p-value is uniformly distributed and finite-sample type-I error control is guaranteed.

\paragraph{Private CRT.} Let $\cD = (\bx^{(0)}, \dots, \bx^{(m)}, \by, \bz)$ denote the aggregated dataset. 
We say $\cD'$ is a neighbor of $\cD$ if they differ in at most one row. By defining $\cD$ to include the resamples $\bx^{(1)}, \dots, \bx^{(m)}$, we also protect the privacy of the data obtained in the resampling step. 

Our private CRT test is shown in \Alg{crt}: it obtains a private estimate of the rank of $T_0$ amongst the statistics $T_1, \dots, T_m$, sorted in decreasing order.  
Using the Laplace mechanism to privately estimate the rank is not a viable option, since the rank has high sensitivity: changing one point in $\cD$ could change all the values $T_0, \dots, T_m$ and change the rank of $T_0$ by $O(m)$. Another straightforward approach is to employ the widely used Sparse Vector Technique \cite{DworkNRRV09, DworkR14} to privately answer questions "Is $T_i > T_0$?" for all $i \in [m]$. However, this algorithm pays a privacy price for each $T_i$ that is above the “threshold” $T_0$, which under the null is $\Omega(m)$, thus resulting in lower utility of the algorithm. Instead, we define a new score function and algorithm which circumvents this problem by intuitively only incurring a cost for the queries $T_i$ that are very close to $T_0$ in value. 

Our key algorithmic idea is to define an appropriate score function of bounded sensitivity. It assigns a score to each rank $c \in [0, m]$ that indicates how well $c$ approximates the true rank of $T_0$. 
The score of a rank $c$ equals the negative absolute difference between $T_0$ and the statistic at rank $c$.  The true rank of $T_0$ has the highest score (equal to $0$), whereas all other ranks have negative scores. 
We show that this score function has bounded sensitivity for statistics $T$ of bounded sensitivity. The rank with the highest score is privately selected using Report Noisy Max, a popular DP selection algorithm~\citep{DworkR14}. 
To obtain good utility, the design of the score function exploits the fact that for CRTs, the values $T_i$ are distributed in a very controlled fashion, as explained in the remark following \Thm{crt}.

\begin{algorithm}
  \caption{Private Conditional Randomization Test} 
  \label{alg:crt}
  \begin{algorithmic}[1]
    \State {\bfseries Input:} Dataset $(\bx^{(0)}, \by, \bz)$, privacy parameter $\eps$, bound $\Delta_T$ on the sensitivity of $T$, number of resamples $m$. 
    \State $T_0 \leftarrow T(\bx^{(0)}, \by, \bz)$, $s_0 \leftarrow 0$. 
    \For{ $i = 1, \dots, m$} 
    \State  Sample $\bx^{(i)} \mid \bz$ from $X \mid Z$ \label{step:gem}.
    \State $T_i \leftarrow T(\bx^{(i)}, \by, \bz)$.
    \EndFor
    \State Let $Q_0, \dots, Q_m$ denote the values $\{T_i\}_{i\in[0,m]}$ sorted in decreasing order.
     \For{ $i = 0, 1 \dots, m$}
    \State $s_i \leftarrow -\frac{|Q_i - T_0|}{2\Delta_T}$.
    \EndFor
    \State $\hat{c} \leftarrow \mathtt{ReportNoisyMax}(\{s_i\}_{i \in [0, m]}, \eps)$. \Comment{\Thm{rnm}}
    \State Output p-value $\hat{p} = \frac{1+\hat{c}}{m+1}$. 
  \end{algorithmic}
\end{algorithm}

\begin{theorem}[Report Noisy Max~\citep{DworkR14, McKennaS20, Ding21}]
\label{thm:rnm}
    Let $\eps > 0$. Given scores $s_i \in \R$, $i \in [B]$ evaluated from a score function of sensitivity at most $1$, the algorithm $\mathtt{ReportNoisyMax}$ samples $Z_1, \dots, Z_B \sim \mathrm{Exp}(2/\eps)$ and returns $\hat{\imath} = \argmax_{i \in [B]} (s_i +Z_i)$.
    This algorithm is $\eps$-DP and for $\delta \in (0,1)$, with probability at least $1 - \delta$, it holds
    $s_{\hat{\imath}} \geq s^* - 2\log(B/\delta)\eps^{-1}$,
where $s^* = \max_{i \in [B]} s_i$. 
\end{theorem}

We describe our score function in \Def{score} and 
bound its sensitivity in \Lem{sensitivity_score}. The bound on the sensitivity of the score function is obtained by assuming a bound $\Delta_T$ on the sensitivity of the statistic $T$.

\begin{definition}[Score function for rank of query] 
\label{def:score}
Let $\{T_i\}_{i \in [0, m]}$ be $m+1$ queries of sensitivity at most $\Delta_T$ on a dataset $\cD$.  Let $Q_0, \dots, Q_m$ denote the values $\{T_i\}_{i \in [0,m]}$ sorted in decreasing order. Let $k \in [0, m]$ be the index of the query whose rank we wish to know. Then for all $c \in [0,m]$, define
\begin{align*}
    s_k(c, \cD) = -\frac{|Q_c - T_k|}{2\Delta_T}.
\end{align*}
\end{definition}

\begin{restatable}{lemma}{sen}{\emph{(Sensitivity of the score function).}}
\label{lem:sensitivity_score}
Let $\{T_i\}_{i \in [0,m]}$ be the values of $m+1$ queries of sensitivity at most $\Delta_T$ on a dataset $\cD$. Let $\{T_i'\}_{i \in [B]}$ be the values of the same queries on a neighboring dataset $\cD'$. Let $Q_0, \dots, Q_m$ (respectively $Q_0', \dots, Q_m'$) denote the values $\{T_i\}_{i \in [0,m]}$ (respectively $\{T_i'\}_{i \in [0,m]}$ ) sorted in decreasing order. Then $|Q_c - Q_c'| \leq \Delta_T$ for all $c \in [0,m]$. As a result, the score function $s_k(c, \cD)$ has sensitivity at most 1 for all $c \in [0,m]$. 
\end{restatable}

\paragraph{Statistic $T$ and its Sensitivity.} The statistic $T$ that we use to obtain our private CRT test is defined as the numerator of the GCM statistic. The residuals of $\by$ with respect to $\bz$ are calculated by fitting a KRR model of $\by$ to $\bz$. Denote such residuals $r_{Y, i}$, for $i \in [n]$. The residuals of $\bx$ with respect to $\bz$ are exact, since we have access to the distribution $X \mid Z$. Denote such residuals $r_{X, i}$ for $i \in [n]$. The residual products are calculated as $R_i = r_{X,i}r_{Y, i}$ for $i \in [n]$.

\begin{definition}[Statistic $T$ for the private CRT]
\label{def:crt_statistic}
    Given a dataset $(\bx, \by, \bz)$ of $n$ points, 
    let $(R_1, \dots, R_n)$ be the vector of residual products of the exact residuals of $\bx$ with respect to $\bz$ and the residuals of fitting a kernel ridge regression model of $\by$ to $\bz$. Define $T(\bx, \by, \bz) = \sum_{i=1}^n R_i$. 
\end{definition}

We obtain a bound of $O_\lambda(1)$ on the $\ell_1$-sensitivity of the statistic $T$ by bounding the sensitivity of the $R_i$. To bound the sensitivity of the  $r_{Y, i}$ we assume that the domain of the variable $Y$ is bounded and use the result of \Thm{kusner}. We assume a known bound on the magnitude of the residuals $r_{X,i}$, motivated by the fact that we have access to the distribution $X \mid Z$. This differs from the assumptions for our PrivGCM test, where we assumed bounds on both $X$ and $Y$. Assuming a bound on the residuals $r_{X,i}$ gives a tighter sensitivity bound for $T$. 

\begin{restatable}{lemma}{sensitivitycrt} \emph{(Sensitivity of $T$).}
\label{lem:sensitivity_crt}
    Consider two neighboring datasets $\cD = (\bx^{(0)}, \dots, \bx^{(m)}, \by, \bz)$ and $\cD' = (\bx'^{(0)}, \dots, \bx'^{(m)}, \by', \bz')$. For $j \in [0, m]$, let $T_j = T(\bx^{(j)}, \by, \bz)$.  Define $T_j'$ analogously. If $|y_i|, |y'_i| \leq 1$ for all $i \in [n]$ and $|r_{X,i}^{(j)}|, |{r'}_{X,i}^{(j)}| \leq 1$ for all $i \in [n], j \in [0, m]$,\footnote{The bound of $1$ can be replaced by any constant.} then $       |T_j - T_j'| \leq \crtconst$,    where $\crtconst = 4\cdot (1 +\frac{\sqrt{2}}{\sqrt{\lambda}} +  \frac{2\sqrt{2}}{\lambda^{3/2}} + \frac{2}{\lambda})$. 
\end{restatable}

\paragraph{Accuracy of the Private CRT.} We define the accuracy of \Alg{crt} in terms of the difference between the private p-value it outputs and its non-private counterpart. 
\begin{definition}[$(\gamma, \delta)$-accuracy]
    \label{def:c_gamma}
    Let $G_{\gamma} = |\{ i \in [m] \mid |T_i - T_0|  \leq \gamma \}|$.  Let $c^*$ be the rank of $T_0$ given statistics $T_i, i \in [0,m]$, and $p^* = (1+c^*)/(m+1)$ be the non-private p-value. We say \Alg{crt} is $(\gamma, \delta)$-accurate if with probability at least $1-\delta$ it holds $|\hat{p} - p^*| \leq \frac{G_\gamma}{m+1}$. 
\end{definition}

Define PrivCRT as \Alg{crt} where $T$ is the statistic from \Def{crt_statistic} and $\Delta_T$, the bound on the sensitivity of the statistic $T$, is as given in \Lem{sensitivity_crt}.

\begin{restatable}{theorem}{crt}
\label{thm:crt}
    \emph{PrivCRT} is $\eps$-DP and $(\gamma, \delta)$-accurate for 
$
    \gamma = \frac{ 4\Delta_T}{\eps}  \log\Big(\frac{m}{\delta} \Big).
$
\end{restatable}
\vspace*{-1ex}
\noindent\textbf{Remark on the Accuracy.} 
While the guarantees from the above theorem does not translate into a meaningful type-I error control bound, empirically, we observe that, under the null, the p-values output by PrivCRT are uniformly distributed (\Fig{pvals}, Appendix~\ref{app:extra_experiments}), and thus the test provides type-I error control. 
Under the alternate, $T_0$ is much larger (or smaller) than all the other values $T_i, i \geq 1$ and thus $G_\gamma$ is small. 
However, the power of PrivCRT can be affected if we increase $m$, as this can increase the value of $G_{\gamma}$ (see \Fig{crt_m}, Appendix~\ref{app:extra_experiments}). An interesting open question is whether the dependence on $m$ in the accuracy of a private CRT is avoidable. For now, we recommend using $m = O(1/\alpha)$ for rejection level $\alpha$.

\section{Empirical Evaluation} \label{sec:experiments}

We evaluate our algorithms on a real-world dataset and synthetic data. We start with the latter as it has the advantage that we know the ground-truth of whether $X \perp\!\!\!\perp Y \mid Z$. 

\paragraph{Setup.} To generate synthetic data, we use a setup similar to that of \citet{ShahP20} proposed for evaluating the performance of GCM. Fix an RKHS $\mathcal{H}$ that corresponds to a Gaussian kernel. The function $f_s(z) = \exp(-s^2/2)\sin(sz)$ satisfies $f_s \in \mathcal{H}$.  $Z = (Z_1, \dots, Z_d)$ is a $d$-dimensional variable, where $d \in \{1, 5\}$. The distribution of $(X, Y, Z)$ is as follows: 
\begin{eqnarray*}
Z_1, \dots, Z_d \sim N_Z, \quad X = f_s(Z_1) + N_X, \quad Y = -f_s(Z_1) + N_Y + \beta \cdot N_X,
\end{eqnarray*}
where $N_Z \sim \cN(0, 4)$, $N_X \sim \cN(0, 1), N_Y \sim \cN(0, 1)$, and $\beta \geq 0$ is a constant controlling the strength of dependence between $X$ and $Y$. If $\beta = 0$, then $X \indep Y \mid Z$, but not otherwise.  For experiments with PrivGCM, the dataset $(\bx, \by, \bz)$ consists of $n$ points sampled as above. For experiments with PrivCRT, we additionally sample $m$ copies $\bx^{(j)}$, $j \in [m]$, by fixing $\bz$.  We study how varying $\beta$, $s$, and $n$ affects the rejection rate of our tests (averaged over 500 resampled datasets). Shaded error bars represent 95\% confidence intervals. 
We set type-I error level $\alpha = 0.05$.

We rescale $X$ and $Y$ so that all datapoints $x_i$ and $y_i$ satisfy $|x_i| \leq 1, |y_i| \leq 1$ (recall that we assume known bounds for the data; for this simulation, standard Gaussian concentration implies $\max_{i \leq n}X_i \leq \sqrt{2\log n}$ with very high probability, so choosing $\sqrt{C\log n}$ suffices here for a sufficiently large constant $C$). 
We then fit a KRR model with a Gaussian kernel of $\bx$ to $\bz$ and $\by$ to $\bz$. The best model is chosen via 5-fold cross-validation and grid search over the regularization parameter $\lambda$ and the parameter of the Gaussian kernel.\footnote{Hyper-parameters are optimized non-privately, as is common in the literature on privately training ML models. Our algorithm can be combined in a black box fashion with methods for performing private hyper-parameter search such as \cite{LiuT19}.} The choice of $\lambda$ requires  balancing the performance of the fitting step of the algorithm with the magnitude of noise added (see \Lem{sensitivity_residuals}), and thus some lower bound on $\lambda$ is needed. We enforce $\lambda \geq 10$ and find that this does not hurt performance of the fitting step even for large $n$.  See \Fig{data} (Appendix~\ref{app:extra_experiments}) for an example.

\paragraph{Comparison to the Private Kendall CI test~\citep{WangPS20} and Test-of-Tests~\citep{KazanSGB23}.} We start with a comparison to two other private CI tests in the literature. The first is the private Kendall's CI test, proposed by~\citet{WangPS20}~for categorical variables. The second, which we call PrivToT, is obtained from the Test-of-Tests framework of \citet{KazanSGB23} and uses the non-private GCM as a black-box. See Appendix~\ref{app:extra_experiments} for details on the implementations of these two tests.
\begin{figure}
\centering
\includegraphics[width = 0.7\textwidth]{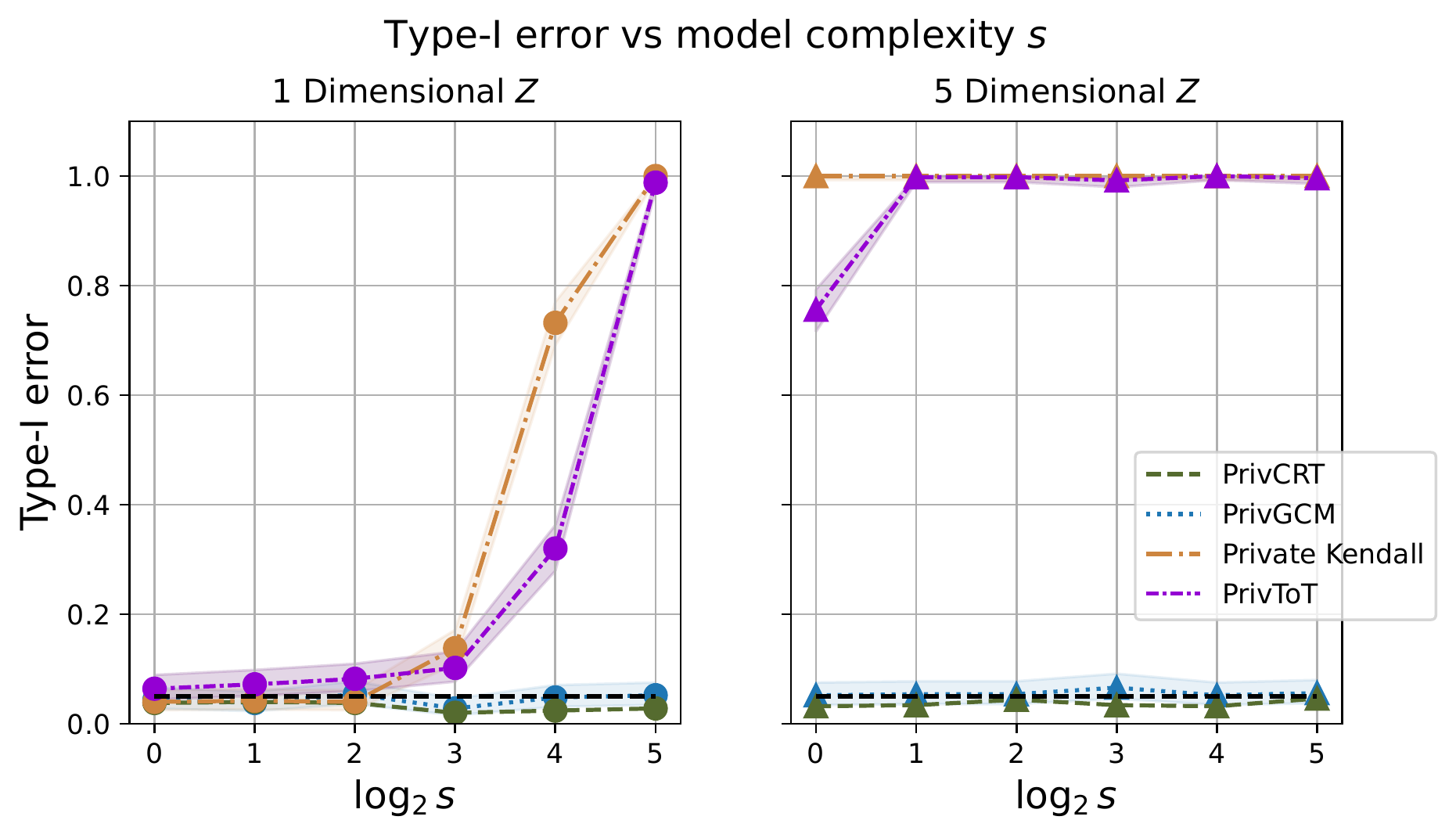}
\caption{Type-I error control of PrivToT, private Kendall, PrivGCM, and PrivCRT (under the null): the first two fail to control Type-I error.}
\label{fig:kendall}
\end{figure}
In \Fig{kendall}, we compare the performance of these two tests with our private tests under the null hypothesis. We vary $s$, the model complexity from $1$ to $32$, and use a sample size $n = 10^4$ and privacy parameter $\eps = 2$. The larger $s$, the harder it is to learn the function $f_s$. As the model complexity increases, the private Kendall test and PrivToT cannot control type-I error, even with a large sample size ($n = 10^4$). They perform even worse when $Z$ is $5$-dimensional. On the other hand, both PrivGCM and PrivCRT have consistent type-I error control across model complexity and dimensionality of $Z$. This experiment motivates the need for tests with rigorous theoretical type-I error guarantees, as we derive.\footnote{Note that for tests without the desired type-I error control, statements about power are vacuous.}


Next, we compare our private CI tests with their non-private counterparts. We fix $s = 2$ for $f_s$.

\begin{figure*}[!ht]
\begin{minipage}{0.4\textwidth}
\centering
\includegraphics[width = 1\textwidth]{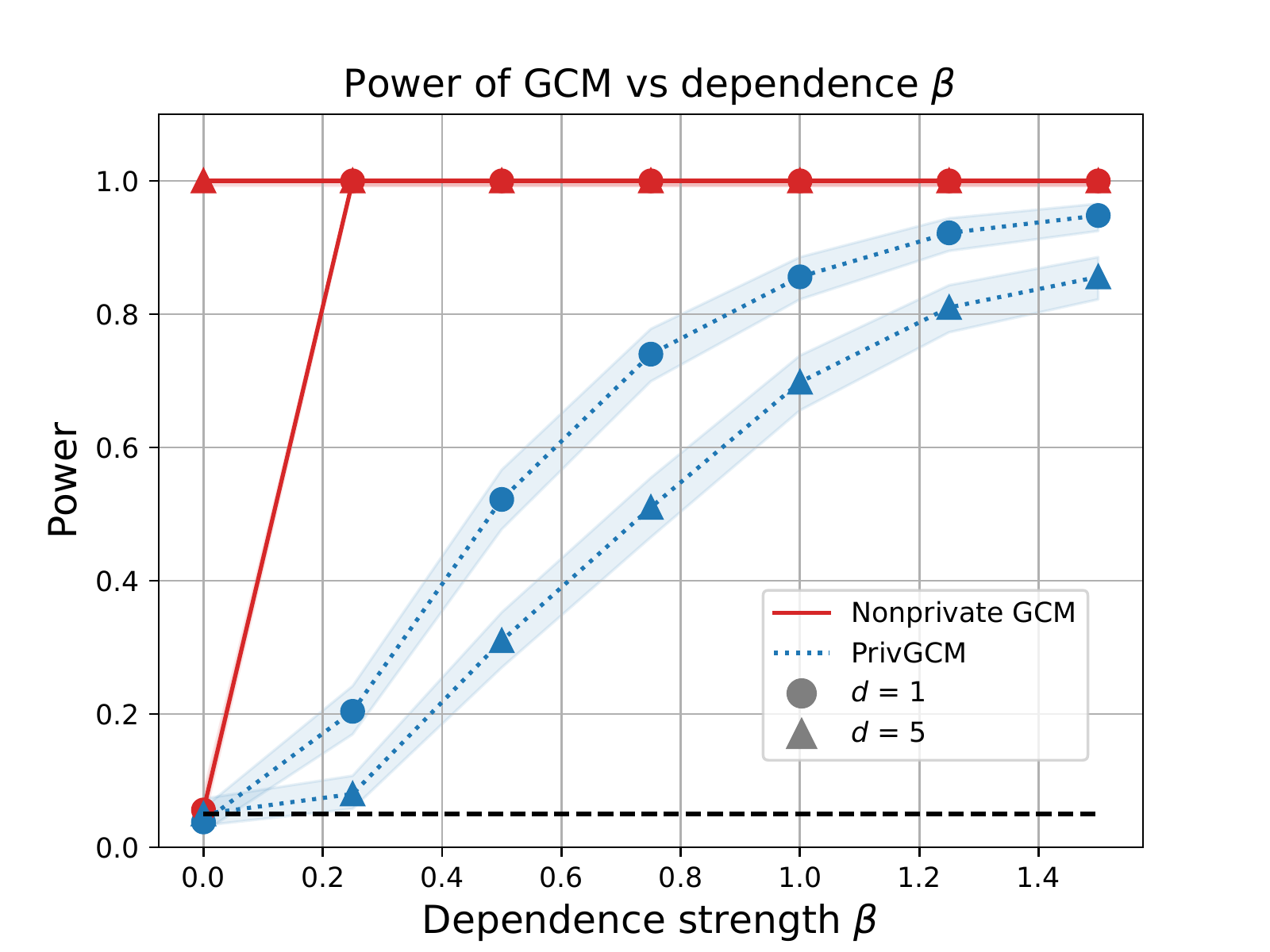}
\caption{Comparison of the power of private and nonprivate GCM tests as the dependence strength $\beta$ increases. At $d=5$, the (nonprivate) GCM fails to provide type-I error control when $\beta=0$.}
\label{fig:betagcm}
\end{minipage}\hfill
\begin{minipage}{0.55\textwidth}
\centering
\includegraphics[width = 1\textwidth]{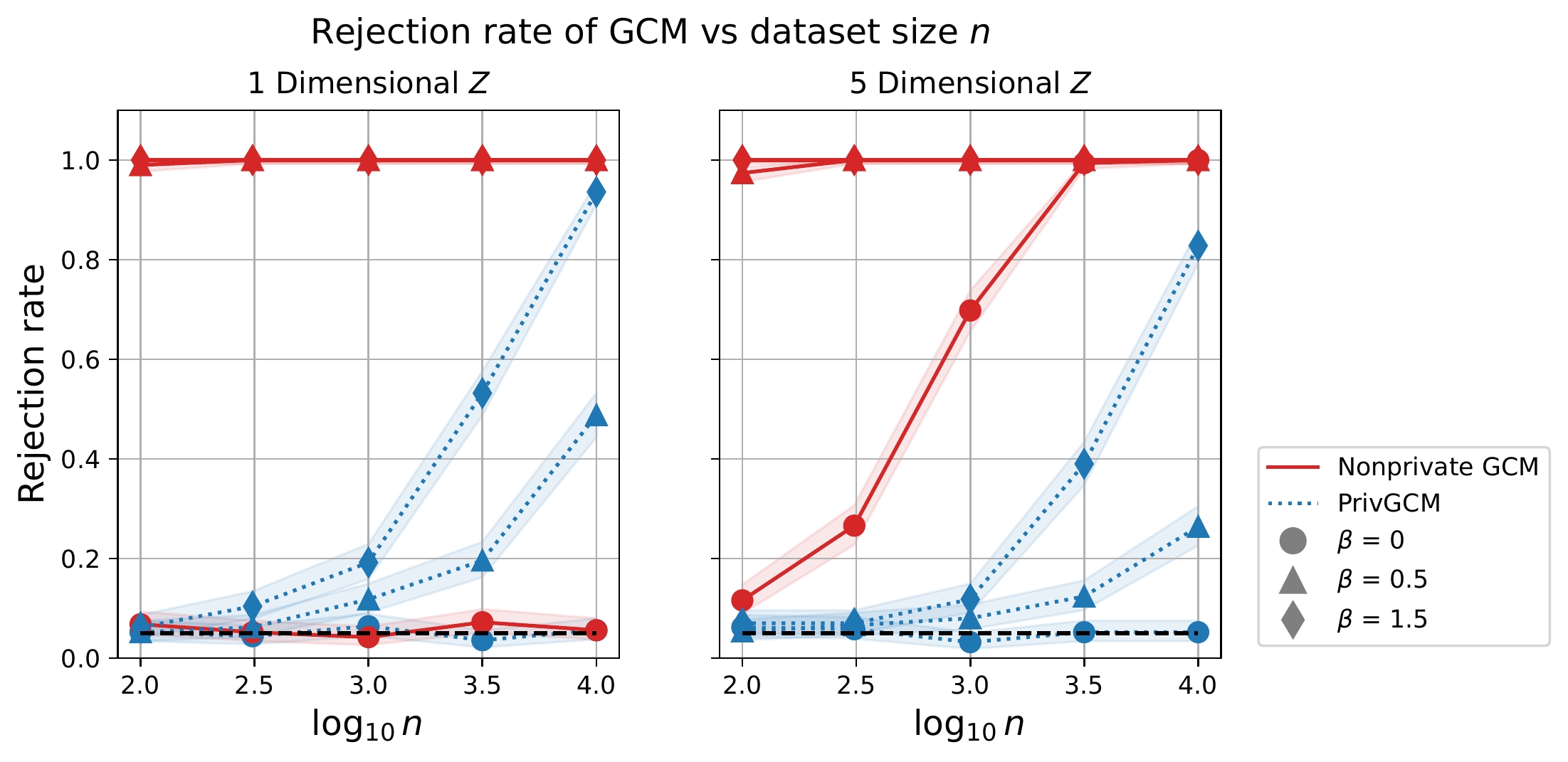}
\caption{Comparison of the type-I error and power of private and nonprivate GCM tests as the dataset size $n$ increases. Again, at $d=5$ with $\beta=0$, the (nonprivate) GCM fails to provide type-I error control even at large $n$ (in fact, its type-I error gets worse with $n$).}
\label{fig:ngcm}
        \end{minipage}
\end{figure*}

\begin{figure*}
\begin{minipage}{0.4\textwidth}
\centering
\includegraphics[width = 1\textwidth]{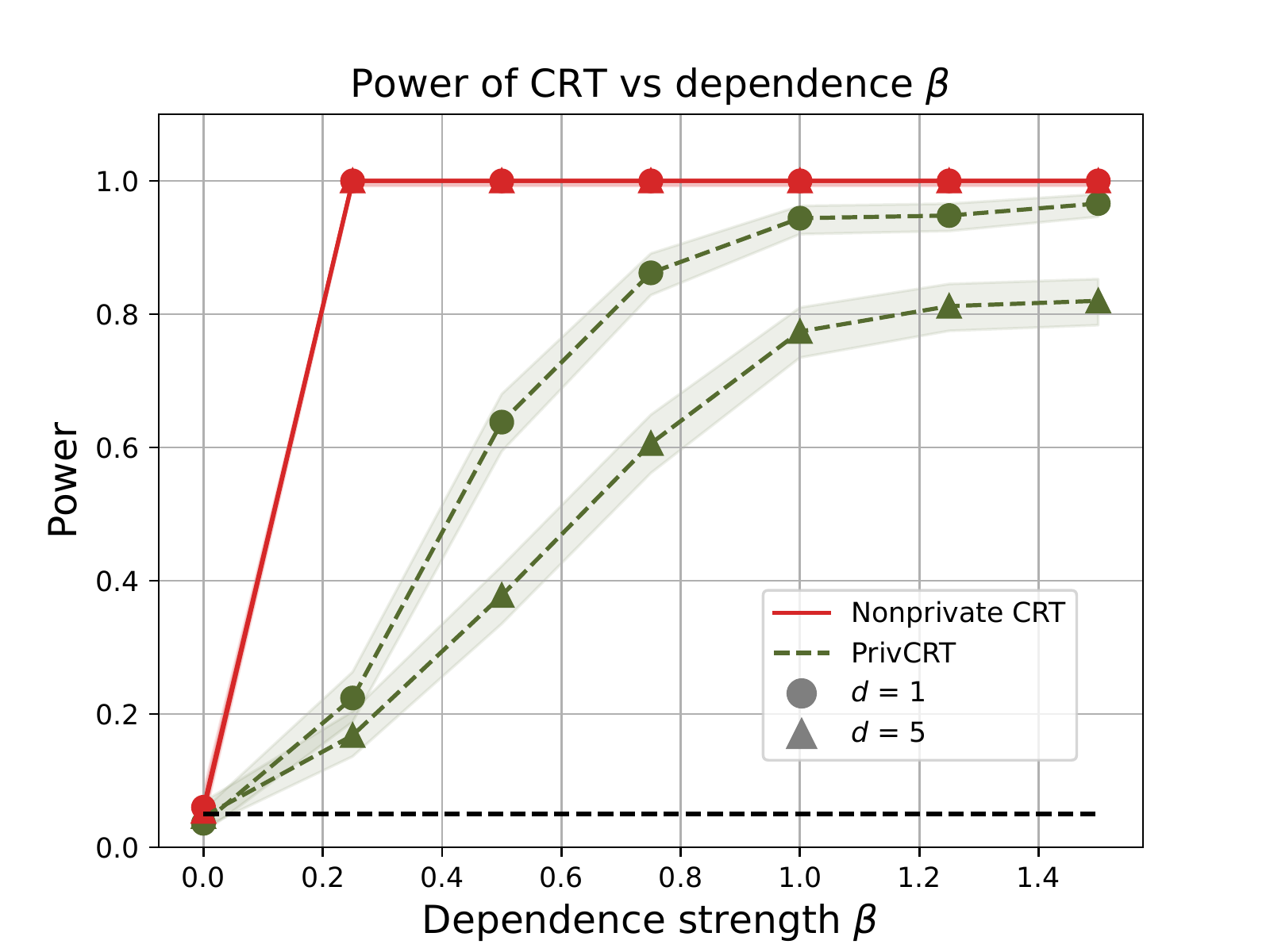}
\caption{Comparing power of private and nonprivate CRT tests as we increase  dependence  $\beta$. }
 \label{fig:betacrt}
    \end{minipage}\hfill
    \begin{minipage}{0.55\textwidth}
\centering
\includegraphics[width = 1\textwidth]{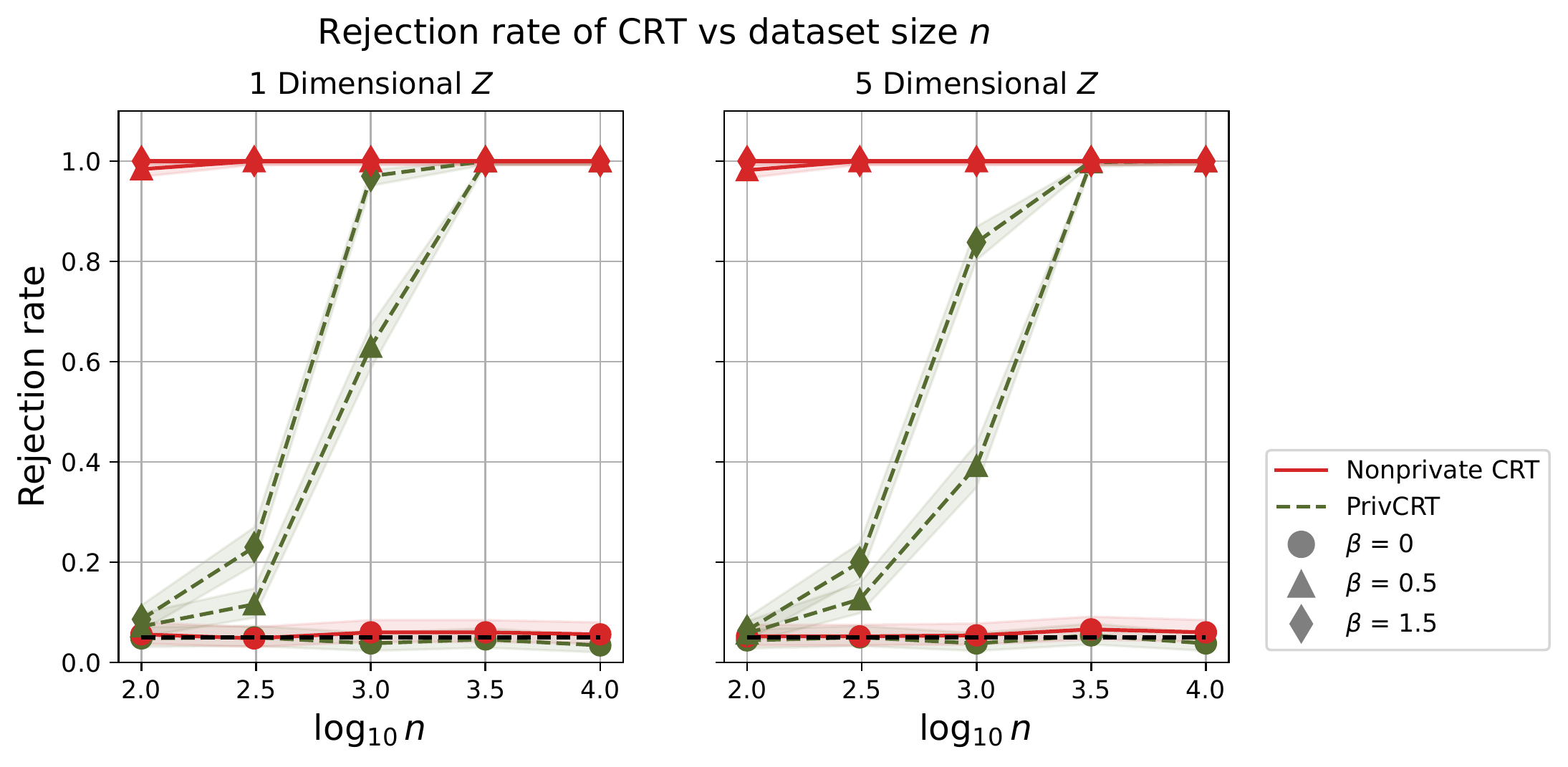}
\caption{Comparison of the type-I error and power of private and nonprivate CRT tests as we increase the dataset size $n$.}
 \label{fig:ncrt}
        \end{minipage}
\end{figure*}

\paragraph{Performance of PrivGCM.}

In \Fig{betagcm}, we vary $\beta$, the strength of the dependence between $X$ and $Y$ from $0$ to $1.5$ and compare the rejection rate of PrivGCM with the (non-private) GCM. We set $n = 10^4$ and privacy parameter $\eps=7$. In the one-dimensional case, i.e., when $d=1$, the rejection rate of both tests goes from $0.05$ to $1$, with the rejection rate of GCM converging faster to 1 than for PrivGCM, as a consequence of the noise added for privacy. Crucially though, when $d=5$, the privacy noise helps PrivGCM provide the critical type-I error control at $\beta=0$, which non-private GCM fails at.
In \Fig{ngcm}, we examine the performance of PrivGCM on synthetic data.

The failure of (non-private) GCM to provide type-I error control is better examined in \Fig{ngcm}, where we vary the dataset size $n$ from $10^2$ to $10^4$, and plot the rejection rate of PrivGCM and GCM for $\beta = 0$ and $\beta > 0$. When $d=5$, the privacy noise helps PrivGCM provide the critical type-I error control at $\beta=0$, which non-private GCM fails at. At $d=5$, the KRR model fails to fit the data (it returns a predicted function that is nearly-zero). In this case, for $\beta = 0$, the GCM statistic converges to a Gaussian of standard deviation 1, but whose mean is removed from zero. 
The larger $n$, the further the mean of the Gaussian is from zero, and the worse the type-I error. The noise added for privacy brings the mean closer to zero since the standard deviation of the noisy residuals, $\sigma_P'$, is much larger than $\sigma_P$ (see \Eqn{power_main}).  
For $\beta > 0$, PrivGCM needs a higher dataset size to achieve the same power as GCM, concordant with our discussion following \Thm{power_rescaled}. As $n \to \infty$, the power of PrivGCM is expected to approach $1$. 

\paragraph{Performance of PrivCRT.}   We study the performance of PrivCRT in Figs.~\ref{fig:betacrt}-\ref{fig:ncrt}. 
PrivCRT achieves better power than PrivGCM for our setup, so we use a smaller privacy parameter of $\eps=2$ and set $m = 19$ (an extreme, but valid choice).
In \Fig{betacrt}, we vary $\beta$, the dependence strength between $X$ and $Y$, from $0$ to $1.5$, using $n = 10^3$. Both non-private CRT and PrivCRT provide type-I error control. Also, the power of both PrivCRT and (non-private) CRT converges to $1$, with a faster convergence for the non-private test.  In \Fig{ncrt}, we vary the dataset size $n$ and $\beta \in \{0, 0.5, 1.5\}$. 




\paragraph{PrivCRT vs.\ PrivGCM.} In \Fig{crt_gcm} we compare the performance of PrivCRT with PrivGCM for different privacy parameters $\eps \in [2^{-3}, 2^3]$. We set $n = 10^4$ and, for PrivCRT, $m=19$. Both tests control type-I error, but PrivCRT achieves better power than PrivGCM for all privacy parameters $\eps$. Therefore, PrivCRT appears preferable to PrivGCM when dataholders have access to the distribution $X \mid Z$. This result is consistent with the non-private scenario where the CRT has higher power because it does not have to learn $\E[X \mid Z]$.

\begin{figure}[!h]
\centering
\includegraphics[width = 0.8\textwidth]{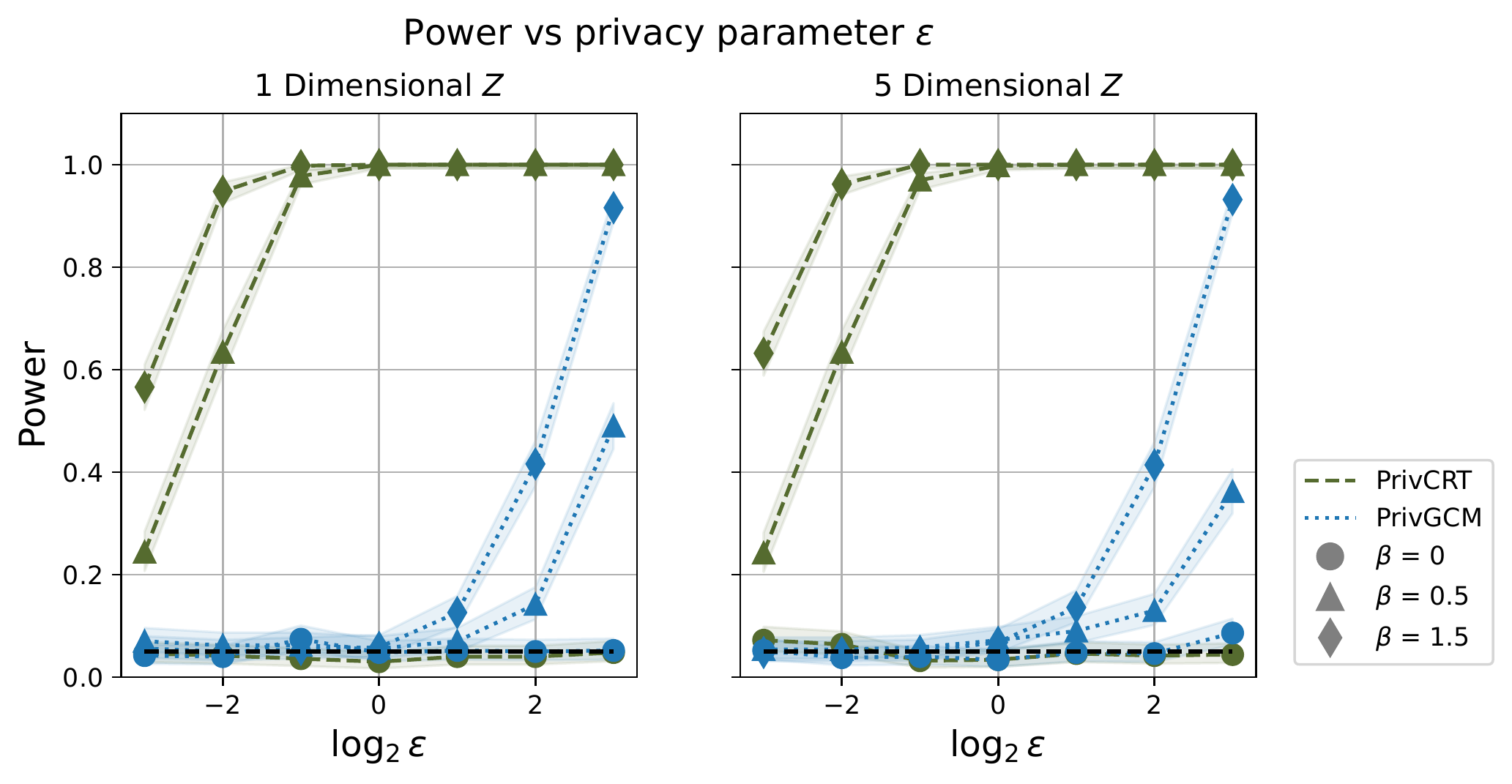}
\caption{Power of PrivCRT and PrivGCM versus privacy $\eps$.}
\label{fig:crt_gcm}
\end{figure}

\paragraph{Real Data Experiments.}
We now demonstrate the performance of our private tests on real-world data. For these experiments, we use the ``Concrete Compressive Strength'' dataset that consists of $1030$ datapoints and $9$ continuous features \cite{concrete}. We set $Y$ to be the outcome feature: the concrete compressive strength. The choice of the $X$ variable is varied, and $Z$ consists of the $7$ remaining features. $Y$ is assumed to be a complex non-linear function of the other $8$ features, thus we expect the CI tests to reject. We focus on the private GCM because among our two algorithms, as it is the more generally applicable one, i.e., it does not require access to the conditional distribution. In \Fig{concrete}, we evaluate the power of the PrivGCM. As desired, the power of PrivGCM tends to $1$ as $n$ increases approaching the power of the the non-private case. 

\begin{figure}[!ht]
\centering
\includegraphics[width = 0.6\textwidth]{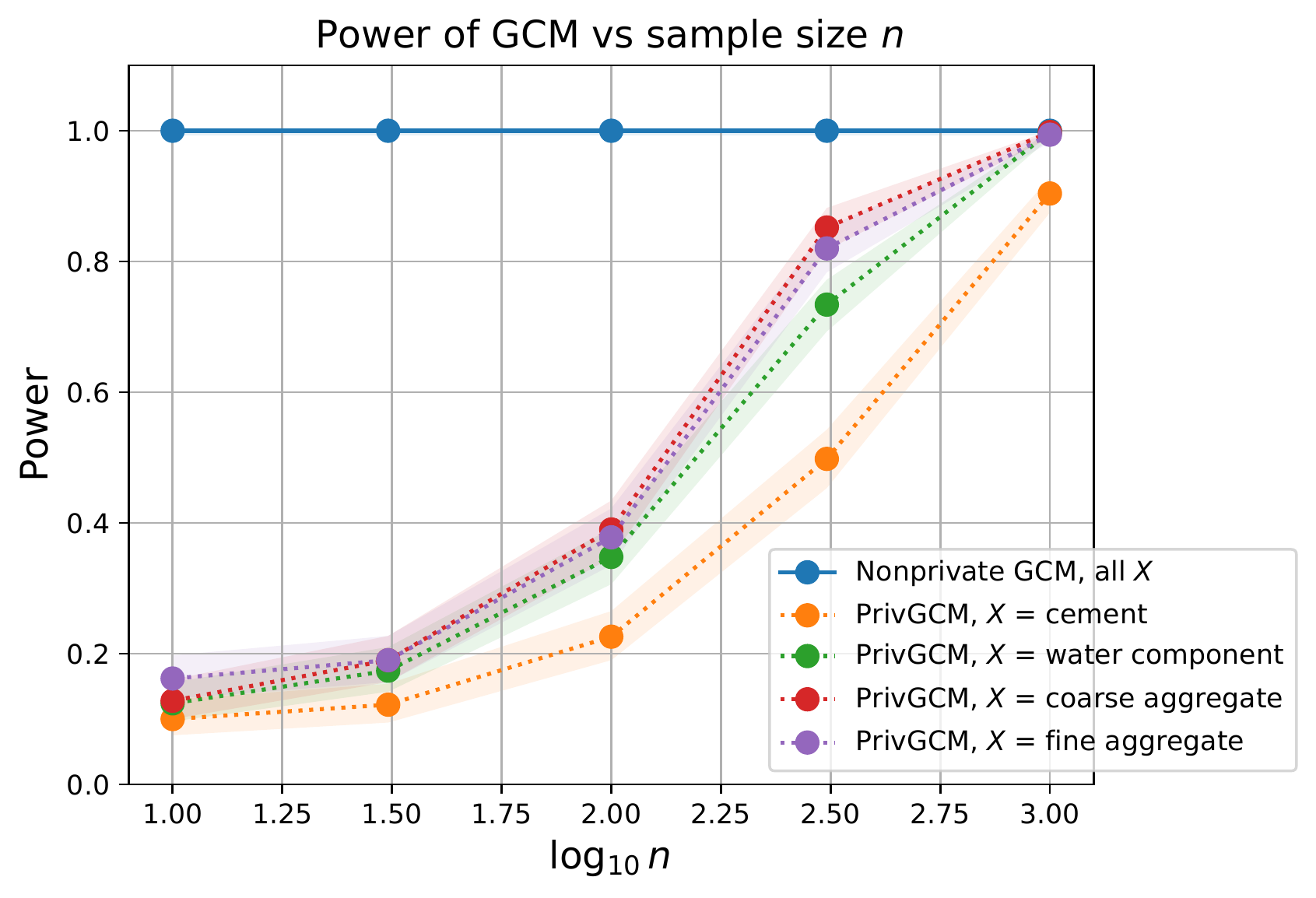}
\caption{Power of the non-private GCM and PrivGCM on the ``Concrete Compressive Strength'' dataset. The power of PrivGCM tends to $1$ with increasing sample size.}
\label{fig:concrete}
\end{figure}

\section{Concluding Remarks} 

This work studies the fundamental statistical task of conditional independence testing under privacy constraints. We design the first DP conditional independence tests that support the general case of continuous variables and have strong theoretical guarantees on both statistical validity and power. Our experiments support our theoretical results and additionally demonstrate that our private tests have more robust type-I error control than their non-private counterparts. 

We envision two straightforward generalizations of our private GCM test. First, our test can be generalized to handle multivariate $X$ and $Y$, following~\citet{ShahP20}, who obtain the test statistic from the residual products of fitting each variable in $X$ and each variable in $Y$ to $Z$. 
 A natural extension would be to compute the same statistic on our noisy residual products. 
Secondly, following \citet{scheidegger2022weighted}, a private version of the weighted GCM would allow the test to achieve power against a wider class of alternatives than the unweighted version. 
Finally, constructing private versions of other model-X based tests, such as the Conditional Permutation Test~\citep{Berrett2019TheCP}, could be another interesting direction.

 \subsection*{Acknowledgements}
We would like to thank Patrick Bl\"obaum for helpful initial discussions surrounding this project.

\bibliographystyle{plainnat}
\bibliography{references}

\newpage
\appendix

\section{Proofs of \Sec{gcm}} \label{app:gcm}

\subsection{Proof of \Thm{level_rescaled}}
 
 In this section, we state and prove a longer version of \Thm{level_rescaled}. Item~\ref{item:1} of \Thm{level_rescaled_full} gives the pointwise asymptotic level guarantee of the private GCM, whereas Item~\ref{item:2} shows the more desirable uniform asymptotic level guarantee under a slightly stronger condition. Item 2 corresponds exactly to \Thm{level_rescaled}.\footnote{The assumptions in Item~\ref{item:2}, \Thm{level_rescaled_full} are identical to those of  Definition~\ref{def:assumption}. The assumptions in Item 1 of \Thm{level_rescaled_full}  are slightly weaker than those of Definition~\ref{def:assumption}.} See \Sec{stats} for definitions of pointwise asymptotic level and uniformly asymptotic level. 

 From a privacy perspective, the proof of Item~\ref{item:2} is more involved. While~\citet{ShahP20} consider the asymptotic behavior of variables $\chi_i \xi_i$ (the product of the true residuals), we instead study the behavior of $\chi_i \xi_i W_i$ (the product of the true residuals with the noise random variables). Similarly, while they study the product of error terms of the fitting method, $(f(z_i) - \hat{f}(z_i))(g(z_i) - \hat{g}(z_i))$, we instead need to consider  $(f(z_i) - \hat{f}(z_i))(g(z_i) - \hat{g}(z_i))W_i$, the product of error terms with the noise variables. A key step in the proof is to show that the noise variables grow at a slower rate than the rate of decay of the error terms, with increasing sample size $n$. 
    
Let $\mathcal{E}_0$ be the set of distributions for $(X, Y, Z)$ that are absolutely continuous with respect to the Lebesgue measure.  The null hypothesis, $\cP_0 \subset \mathcal{E}_0$, is the subset of distributions for which $X \indep Y \mid Z$.  
Given $P \in \cP$, let $P'$ be the joint distribution of variables $(X, Y, Z, W)$ where $W \sim \Lap(\Delta/\eps)$ is independent of $(X,Y,Z)$. For a set of distributions $\cP$, let $\cP'$ denote the set of distributions $P'$ for all $P \in \cP$. Denote by $\Phi$ the CDF of the standard normal distribution. 
 
\begin{restatable}{theorem}{levelrescaled1}{\emph{(Type-I Error Control of Private GCM)}} \label{thm:level_rescaled_full}
Let $a$ and $b$ be known bounds on the domains of $X$ and $Y$, respectively. Let $\cP_0 \subset \mathcal{E}_0$ be the set of null distributions defined above. Given a dataset $\cD = (\bx, \by, \bz)$,  let $(\hat{\bx}, \hat{\by}, \bz)$ be the rescaled dataset obtained by setting $\hat{\bx} = \bx/ a$ and $\hat{\by} = \by /b$. Consider $R_i, i \in [n]$, as defined in \eqref{res_products}, for the rescaled dataset $(\hat{\bx}, \hat{\by}, \bz)$. Let $W_i \sim \mathrm{Lap}(0, \Delta/\eps)$ for $i \in [n]$, where $\Delta, \eps > 0$ are constants. Then $T^{(n)} = T(R_1 + W_1, \dots, R_n + W_n)$, defined in Algorithm~\ref{alg:gcm}, satisfies:
    \begin{enumerate}
        \item \label{item:1} For $P \in \cP_0$ such that $A_f A_g = o_P(n^{-1}), B_f =  o_P(1), B_g = o_P(1)$, and $\E[\chi_P^2 \xi_P^2] < \infty$, then 
         \begin{align*}
            \lim_{n \to \infty}  \sup_{t \in \R} |\Pr_{P'}[T^{(n)} \leq t] - \Phi(t)| = 0.
        \end{align*}
        \item \label{item:2} Let $\cP \subset \cP_0$ be a set of distributions such that $A_f A_g = o_{\cP}(n^{-1}), B_f =  o_{\cP}(1)$ and  $B_g = o_{\cP}(1)$. If in addition $\sup_{P \in \cP} \E[|\chi_P \xi_P|^{2+\eta}] \leq c$,  for some constants $c, \eta > 0$, then
\begin{align}\label{eq:level}
            \lim_{n \to \infty} \sup_{P' \in \cP'} \sup_{t \in \R} |\Pr_{P'}[T^{(n)} \leq t] - \Phi(t)| = 0.
        \end{align}
    \end{enumerate}
\end{restatable}
\begin{proof}
Let $\sigma_P = \sqrt{\Var(\chi_P\xi_P)}$, $\sigmapriv = \Delta/\eps$, and $\sigmajoint = \sqrt{\frac{\sigma_P^2}{a^2b^2} + 2\sigmapriv^2}$.  Denote by $\tau_N$ the numerator of $T^{(n)}$ and by $\tau_D$ the denominator. We sometimes omit $P$ from the notation for ease of presentation. 

\paragraph{Item~\ref{item:1}:}  We first show Item~\ref{item:1}. Specifically, we show that $\tau_N \to \cN(0, \sigmajoint^2)$ and $\tau_D \to \sigmajoint$, which, by Slutsky's lemma (\Lem{slutksy}), would imply that $T^{(n)} \to \cN(0,1)$. For $i \in [n]$, let $\chi_{i} = x_i - f(z_i)$ and $\xi_{i} = y_i - g(z_i)$. Note that
\begin{align}
    \tau_N &= \frac{1}{\sqrt{n}}\sum_{i=1}^n (R_i + W_i) \nonumber \\ 
    &= \frac{1}{ab\sqrt{n}}\sum_{i=1}^n(f(z_i) - \hat{f}(z_i))(g(z_i) - \hat{g}(z_i)) + \frac{1}{ab\sqrt{n}}\sum_{i=1}^n(f(z_i) - \hat{f}(z_i))\xi_{i} \nonumber \\ 
    &+ 
    \frac{1}{ab\sqrt{n}}\sum_{i=1}^n(g(z_i) - \hat{g}(z_i))\chi_{i} + 
    \frac{1}{\sqrt{n}}\sum_{i=1}^n\Big(\frac{\chi_i\xi_i}{ab} + W_i\Big) \label{eq:big_sum_1}
\end{align}
We use the following claim obtained from the proof of Theorem 6 of \citet{ShahP20}.
\begin{claim}[\cite{ShahP20}]\label{clm:r_i_rescaled}
    Under the assumptions listed in Item 1 of \Thm{level_rescaled_full}, the following hold
    \begin{enumerate}
        \item $ \frac{1}{\sqrt{n}}\sum_{i=1}^n(f(z_i) - \hat{f}(z_i))(g(z_i) - \hat{g}(z_i)) \pto 0$. \label{item:a}
        \item $\frac{1}{\sqrt{n}}\sum_{i=1}^n(f(z_i) - \hat{f}(z_i))\xi_{i} \pto 0$. \label{item:b}
        \item  $\frac{1}{\sqrt{n}}\sum_{i=1}^n(g(z_i) - \hat{g}(z_i))\chi_{i} \pto 0$.   \label{item:c}  
         \item $\frac{1}{n}\sum_{i=1}^n R_i \pto 0$.\label{item:d}
        \item $\frac{1}{n}\sum_{i=1}^n R_i^2 \pto \frac{\sigma_P^2}{a^2b^2}$. \label{item:e}
    \end{enumerate}
        Additionally, under the assumptions of Item 2, all convergence statements above are uniform over $\cP$.  
\end{claim}

By \Clm{r_i_rescaled}, Items~\ref{item:a}-\ref{item:c}, the first three terms of the sum in \Eqn{big_sum_1} are $o_P(1)$. We show that the last term of the sum converges to $\cN(0, \sigmajoint^2)$. Note that $\E[\frac{\chi_i\xi_i}{ab} + W_i] = 0$ and $\Var(\frac{\chi_i\xi_i}{ab} + W_i) = \Var(\chi_i \xi_i) + \Var(W_i) = \sigmajoint^2$. Since $\sigmapriv$ is a constant and $\sigma_P < \infty$, we can apply the Central Limit Theorem to obtain the desired convergence to a Gaussian. By Slutsky's lemma, we obtain that $\tau_N \to \cN(0, \sigmajoint^2)$.

We now consider $\tau_D$. Since $\E[W_i] = 0$ for all $i \in [n]$, by the Weak Law of Large Numbers it holds that $\frac{1}{n}\sum_{i=1}^n W_i \to 0$. From Item~\ref{item:d} of \Clm{r_i_rescaled}, we obtain $\frac{1}{n}\sum_{i=1}^n (R_i + W_i) \pto 0$. It remains to show that $\frac{1}{n}\sum_{i=1}^n (R_i+W_i)^2 \pto \sigma_P$. We have
\begin{align*}
    \frac{1}{n}\sum_{i=1}^n (R_i+W_i)^2 =  \frac{1}{n}\sum_{i=1}^n R_i^2 + \frac{2}{n} \sum_{i=1}^n R_i W_i + \frac{1}{n}\sum_{i=1}^{n}W_i^2. 
\end{align*}
Since $R_i$ and $W_i$ are independent, we have $\E[R_i W_i] = \E[R_i]\E[W_i] = 0$. Another application of the Weak Law of Large Numbers yields $\frac{2}{n} \sum_{i=1}^n R_i W_i \pto 0$. Finally, let $B^{(n)} = \frac{1}{n}\sum_{i=1}^{n}W_i^2$. We have $\E[W_i^2] = \Var[W_i] = 2\sigmapriv^2$. The Weak Law of Large numbers gives that $B^{(n)} \to 2\sigmapriv^2$. By Item~\ref{item:e} of \Clm{r_i_rescaled} and Slutsky's lemma, we obtain that $\tau_D \to \sigmajoint$, as desired. This concludes the proof of Item~\ref{item:1}. 

\paragraph{Item~\ref{item:2}:} Turning to Item~\ref{item:2}, more care is needed to replace statements of convergence with statements of uniform convergence due to the presence of privacy noise. We first show that $\tau_N$ converges to $\cN(0, \sigmajoint^2)$ uniformly. Consider again the sum in \Eqn{big_sum_1}. By \Clm{r_i_rescaled}, Items~\ref{item:a}-\ref{item:c}, the first three terms of the sum in \Eqn{big_sum_1} are $o_{\cP}(1)$, and as a consequence, $o_{\cP'}(1)$. We show that
\begin{align}
     \frac{1}{\sqrt{n}}\sum_{i=1}^n\Big( \frac{\chi_i\xi_i}{ab} + W_i \Big) \to \cN\Big(0, \sigmajoint^2 \Big) \text{ uniformly over } \cP'.\label{eq:gaussian_convergence}
\end{align}
Then, applying Item 1 of \Lem{slutsky_uniform} we would obtain the desired convergence of $\tau_N$. We prove \Eqn{gaussian_convergence} by showing that the random variables $\frac{\chi_i\xi_i}{ab} + W_i, i \in [n]$ satisfy the conditions of \Lem{clt_uniform}. 
Recall that $\E[\frac{\chi_i\xi_i}{ab} + W_i] = 0$ and $\E[(\frac{\chi_i\xi_i}{ab} + W_i)^2] = \Var(\frac{\chi_i\xi_i}{ab} + W_i)= \sigmajoint^2$, which is a constant. It remains to bound the $(2+\eta)$-absolute moment of $\frac{\chi_i\xi_i}{ab} + W_i$. By \Clm{moment_of_sum} and our assumption on $\chi_i \xi_i$,  we have 
\begin{align*}
    \E\Big[\Big|\frac{\chi_i\xi_i}{ab} + W_i\Big|^{2+\eta}\Big]  \leq 2^{1+\eta}\Big(\E\Big[\Big|\frac{\chi_i\xi_i}{ab}\Big|^{2+\eta}\Big] + \E[|W_i|^{2+\eta}]\Big) \leq 2^{(1+\eta)} \cdot \Big(\frac{c}{(ab)^{2+\eta}} + c'\Big)
\end{align*}
where $c'$ is the (constant) bound on $\E[|W_i|^{2+\eta}])$ given by \Clm{laplace_moment}. Thus, all the conditions of \Lem{clt_uniform} are satisfied, and \Eqn{gaussian_convergence} holds. This concludes our argument on the convergence of $\tau_N$. 

We now show that $\tau_D$ converges uniformly over $\cP'$ to $\sigmajoint$. 
We first show that $n^{-1}\sum_{i=1}^n (R_i + W_i) = o_{\cP'}(1)$. By Item~\ref{item:d} of \Clm{r_i_rescaled}, we have that $n^{-1}\sum_{i=1}^n R_i = o_{\cP'}(1)$. We also showed that $\frac{1}{n}\sum_{i=1}^n W_i = o_{\cP'}(1)$. Applying \Lem{slutsky_uniform}, we obtain that $n^{-1}\sum_{i=1}^n (R_i + W_i) = o_{\cP'}(1)$. 

Next, we show that $n^{-1}\sum_{i=1}^n R_iW_i = o_{\cP'}(1)$. Note that
\begin{align}
    \frac{1}{n}\sum_{i=1}^n R_iW_i 
    &= \frac{1}{ab} \Big( \frac{1}{ n}\sum_{i=1}^n(f(z_i) - \hat{f}(z_i) + \chi_i)(g(z_i) - \hat{g}(z_i) + \xi_i)W_i \Big) \nonumber\\
    &= \frac{1}{a b} \Big(  \frac{1}{n}\sum_{i=1}^n(f(z_i) - \hat{f}(z_i))(g(z_i) - \hat{g}(z_i))W_i + \frac{1}{ n}\sum_{i=1}^n(f(z_i) - \hat{f}(z_i))\xi_{i} W_i \nonumber \\ &+ 
    \frac{1}{n}\sum_{i=1}^n(g(z_i) - \hat{g}(z_i))\chi_{P, i} W_i + 
    \frac{1}{n}\sum_{i=1}^n\chi_i\xi_i W_i \Big) \label{eq:big_sum}
\end{align}

We show that each of the terms of the sum in \Eqn{big_sum} is $o_{\cP'}(1)$, starting with $n^{-1}\sum_{i=1}^n\chi_i\xi_i W_i$. We do so by showing that the conditions of \Lem{wlln_uniform} hold for the random variables $\chi_{i}\xi_{i} W_i$. More precisely, we need to show that for all $P' \in \cP'$ it holds that $\E_{P'}[|\chi_{i}\xi_{i} W_i|^{1+\eta}] < c_2$ for some $c_2, \eta > 0$. By the assumptions in Item~\ref{item:2}, we know $\E_{P'}[|\chi_P \xi_P|^{1+\eta}] \leq\E_{P'}[|\chi_P \xi_P|^{2+\eta}] \leq c$. By the independence of $\chi_i\xi_i$ and $W_i$ we have 
\begin{align*}
    \E_{P'}[|\chi_i \xi_i W_i|^{1+\eta}] = \E_{P'}[|\chi_P \xi_P|^{1+\eta}]\E_{P'}[|W_i|^{1+\eta}] \leq c \cdot c',
\end{align*}
where $c'$ is a bound on $\E_{P'}[|W_i|^{1+\eta}]$ given by \Clm{laplace_moment}. 
Therefore, we have that the variables $\chi_i\xi_iW_i$ satisfy all properties listed in  \Lem{wlln_uniform}, and thus 
\begin{align}
    n^{-1}\sum_{i=1}^n \chi_i\xi_iW_i= o_{\cP}(1).
\end{align}
Next, we show that 
\begin{align}
    \frac{1}{n}\sum_{i=1}^n(f(z_i) - \hat{f}(z_i))(g(z_i) - \hat{g}(z_i))W_i = o_{\cP'}(1).\label{eq:w_i}
\end{align}
Denote the term in the LHS of \Eqn{w_i} by $b$. We have
\begin{align}\label{eq:product}
     |b| \leq  \max_{i \in [n]}|W_i| \cdot \Big( \frac{1}{n}\sum_{i=1}^n(f(z_i) - \hat{f}(z_i))(g(z_i) - \hat{g}(z_i)) \Bigr)
\end{align}
We show that $\max_{i \in [n]}|W_i| = O_{\cP'}(\log n)$ in \Clm{max_laplace}. By Item~\ref{item:a} of \Clm{r_i_rescaled}, the other term in the product in the RHS of \Eqn{product} is $o_{\cP}(n^{-1/2})$.  As a result, $b = o_{\cP'}(1)$. By a similar proof, and using \Clm{r_i_rescaled}, the other terms in the sum in \Eqn{big_sum} are equal to $o_{\cP'}(1)$. 

By Item~\ref{item:e} of \Clm{r_i_rescaled}, we have that $\frac{1}{n}\sum_{i=1}^{n}R_i^2$ converges uniformly over $\cP'$ to $\frac{\sigma_P^2}{a^2 b^2}$. In addition, we showed that $\frac{1}{n}\sum_{i\in[n]}W_i^2$ converges uniformly over $\cP'$ to $2\sigmapriv^2$. 
Applying \Lem{slutsky_uniform} we obtain that $\tau_D$ converges uniformly over $\cP'$ to $\sigmajoint$. One final application of \Lem{slutsky_uniform} yields that $T^{(n)}$ converges uniformly over $\cP'$ to $\cN(0, 1)$. 
\end{proof}

\subsection{Proof of \Thm{power_rescaled}}
In this section, we state and prove a longer version of \Thm{power_rescaled}. Item~\ref{item:pointwisepower} of \Thm{power_rescaled_full} gives the pointwise power guarantee, whereas Item~\ref{item:uniformpower}shows the uniform  power guarantee under a slightly stronger condition. Item~\ref{item:uniformpower} corresponds exactly to \Thm{power_rescaled}.\footnote{The assumptions in Item 2, \Thm{power_rescaled_full} are identical to those of  Definition~\ref{def:assumption}. The assumptions in Item 1 of \Thm{power_rescaled_full}  are slightly weaker than those of Definition~\ref{def:assumption}.}

Following \citet{ShahP20}, to facilitate the theoretical analysis of power, we separate the model fitting step from the calculation of the residuals. We calculate $\hat{f}$ and $\hat{g}$ on the first half of the dataset and calculate the residuals $R_i, i \in [n]$ on the second half. In practice, it is still advised to perform both steps on the full dataset. 

\begin{restatable}{theorem}{powerrescaled1}{\emph{(Power of Private GCM).}} \label{thm:power_rescaled_full}
    Consider the setup of \Thm{level_rescaled}. Let $A_f, A_g, B_f, B_g$ be as defined in \Eqn{a_and_b}, with the difference that $\hat{f}$ and $\hat{g}$ are estimated on the first half of the dataset $(\hat{\bx}, \hat{\by}, \bz)$, and $R_i, i \in [n/2, n]$ are calculated on the second half. Define the ``signal'' ($\rho_P$) and ``noise'' ($\sigma_P$) of the true residuals $\chi_P, \xi_P$ as: 
    \begin{align*}
        \rho_P = \E_P[\cov(X, Y \mid Z)], \sigma_P = \sqrt{\mathrm{Var}_P(\chi_P\xi_P)}.
    \end{align*}
    \begin{enumerate}
        \item \label{item:pointwisepower} If for $P \in \mathcal{E}_0$ we have $A_f A_g = o_P(n^{-1}), B_f =  o_P(1), B_g = o_P(1)$ and $\sigma_P < \infty $, then 
         \begin{align} 
            \lim_{n \to \infty}  \sup_{t \in \R} \Bigl| \Pr_{P'} \Bigl[ T^{(n)} - \frac{\sqrt{n}\rho_P}{\sigma_P'} \leq t\Bigr] - \Phi(t) \Bigr| = 0, \mbox{ where }  \sigma_P' =  \sqrt{ \sigma_P^2 + (\frac{\sqrt{2}ab\Delta}{\eps} )^2}.   \label{eq:power}
        \end{align}
    \item \label{item:uniformpower} Let $\cP \subset \mathcal{E}_0$ such that $A_f A_g = o_{\cP}(n^{-1}), B_f =  o_{\cP}(1)$ and $B_g = o_{\cP}(1)$. If in addition  $\sup_{P \in \cP} \E[|\chi_P \xi_P|^{2+\eta}] \leq c$, for some constants $c, \eta > 0$, then \Eqn{power} holds over $\cP'$ uniformly. 
    \end{enumerate}  
\end{restatable}
\begin{proof}
Note that $\E[\frac{\chi_P \xi_P}{ab}] =\frac{\rho}{ab}$ and $\sqrt{\Var(\frac{\chi_P\xi_P}{ab})} = \frac{\sigma_P}{ab}$.  The proof is similar to that of \Thm{level_rescaled_full}, using Claim~\ref{claim:power_shahpeters} below from \citet{ShahP20}. 
\end{proof}

\begin{claim}\label{claim:power_shahpeters}
        Under the assumptions listed in \Thm{power_rescaled}, as $n \to \infty$, the following hold. 
           \begin{enumerate}
        \item $ \frac{1}{\sqrt{n}}\sum_{i=1}^n(f(z_i) - \hat{f}(z_i))(g(z_i) - \hat{g}(z_i)) \pto 0$. 
        \item $\frac{1}{\sqrt{n}}\sum_{i=1}^n(f(z_i) - \hat{f}(z_i))\xi_{i} \pto 0$. 
        \item  $\frac{1}{\sqrt{n}}\sum_{i=1}^n(g(z_i) - \hat{g}(z_i))\chi_{i} \pto 0$. 
         \item $\frac{1}{n}\sum_{i=1}^n R_i \pto 0$.
        \item $\frac{1}{n}\sum_{i=1}^n R_i^2 \pto \frac{\sigma_P^2}{a^2b^2}$. 
    \end{enumerate}
    Additionally, under the assumptions in Item 2 of \Thm{power_rescaled}, all convergence statements above are uniform over $\cP$.  
    \end{claim}

\begin{corollary}\label{cor:power_of_1}
Under the assumptions of \Thm{power_rescaled_full}, \Alg{gcm} has asymptotic power of $1$ if $\rho_P \neq 0$.

Next, we show that \Thm{power_rescaled} implies that the private GCM has asymptotic power of $1$. A similar claim and proof holds for uniformly asymptotic power. 
\end{corollary}
\begin{proof}
       Note that \Alg{gcm} has asymptotic power of $1$ if, for all $M > 0$, it holds $\Pr_{P}[T^{(n)} > M] \to 1$ as $n \to \infty$. Given $M > 0$, note that
    \begin{align*}
        \Pr[T^{(n)} \leq M] &= \Pr_{P'} \Bigl[ T^{(n)} - \frac{\sqrt{n}\rho_P}{\sigma_P'} \leq M - \frac{\sqrt{n}\rho_P}{\sigma_P'} \Bigr] 
        \to \Phi\Bigl(M - \frac{\sqrt{n}\rho_P}{\sigma_P'} \Bigr), 
    \end{align*}
    where the convergence statement follows from \Thm{power_rescaled}. 
    Since $\rho_P \neq 0$ and  $\sigma_P'$ is a constant, then $M - \frac{\sqrt{n}\rho_P}{\sigma_P'} \to -\infty$ as $n\to \infty$. Therefore $\Phi\Bigl(M - \frac{\sqrt{n}\rho_P}{\sigma_P'} \Bigr) \to 0$, and as a result $\Pr[T^{(n)} \leq M] \to 0$, as desired. 
\end{proof}

\subsection{Guarantees of PrivGCM}

In this section, we prove Lemma~\ref{lem:sensitivity_residuals} on the sensitivity of the residuals products for kernel ridge regression.
We then use Lemma~\ref{lem:sensitivity_residuals} to prove Corollary~\ref{cor:privgcm} on the type-I error and power gurantess of PrivGCM.

\sensitivityresiduals*
\begin{proof}
     Consider two neighboring datasets $(\bx, \by, \bz)$ and $(\bx', \by', \bz')$.
    For $i \in [n]$, let $r_{X,i}, r_{Y,i}$ denote the residuals of fitting a kernel ridge regression model of $\bx$ to $\bz$ and $\by$ to $\bz$, respectively. Suppose without loss of generality that $(\bx, \by, \bz)$ and $(\bx', \by', \bz')$ differ only in the last datapoint, i.e., $(x_i, y_i,  z_i) = (x_i', y_i', z_i')$ for $i \in [n-1]$. Then, by \Thm{kusner}, for $i\in[n-1]$, we have
    \begin{align*}
        |r_{X,i} - r_{X,i}'| = |(x_i - \bw^\top \phi(z_i)) - (x_i' - \bw'^\top \phi(z_i'))| = |\bw^\top \phi(z_i) -  \bw'^\top \phi(z_i)| \leq \frac{8\sqrt{2}}{\lambda^{3/2}n} + \frac{8}{\lambda n}. 
    \end{align*}
    For the last datapoint we have
      \begin{align*}
        |r_{X,n} - r_{X,n}'| &= |(x_n - \bw^\top \phi(z_n)) - (x_n' - \bw'^\top \phi(z_n'))| \\
        &\leq |x_n - x_n'| + \lVert \bw' \rVert_{\mathcal{H}}|\phi(z_n') - \phi(z_n)| + |\bw'^\top \phi(z_n) - \bw^\top \phi(z_n)| \\
        &\leq 2 + \frac{2\sqrt{2}}{\sqrt{\lambda}} + \frac{8\sqrt{2}}{\lambda^{3/2}n} + \frac{8}{\lambda n}.
    \end{align*}
    Finally, note that for all $i \in [n]$, we have$|r_{X, i}| \leq |x_i| + \lVert \bw \rVert_{\mathcal{H}}\lVert \phi(z_i) \rVert_{\mathcal{H}} \leq 1 + \frac{\sqrt{2}}{\sqrt{\lambda}}$. The same bound holds for $|r_{Y,i}|$. Let $c_1 = 2 + \frac{2\sqrt{2}}{\sqrt{\lambda}}$ and $c_2 = \frac{8\sqrt{2}}{\lambda^{3/2}} + \frac{8}{\lambda}$.  This gives us that for all $i \in [n-1]$:
    \begin{align*}
        |R_i - R_i'| \leq |r_{X, i}|||r_{Y, i} - r_{Y, i}'| + |r_{Y, i}'||r_{X, i} - r_{X,i}'| \leq 2\Big(1 + \frac{\sqrt{2}}{\sqrt{\lambda}}\Big)\Big( \frac{8\sqrt{2}}{\lambda^{3/2}n} + \frac{8}{\lambda n} \Big) = \frac{c_1c_2}{n}.
    \end{align*}
    For $i = n$ we have
    \begin{align*}
        |R_n - R_n'| \leq 2\Big(1 + \frac{\sqrt{2}}{\sqrt{\lambda}}\Big)\Big( 2 + \frac{2\sqrt{2}}{\sqrt{\lambda}} + \frac{8\sqrt{2}}{\lambda^{3/2}n} + \frac{8}{\lambda n} \Big) = c_1(c_1+\frac{c_2}{n}).
    \end{align*}
    Finally,
    \begin{align*}
        \lVert \bR - \bR' \rVert_1 = \sum_{i=1}^n |R_i - R_i'| \leq (n-1) \cdot \frac{c_1c_2}{n} + c_1(c_1+\frac{c_2}{n}) = c_1^2 + c_1c_2,  
    \end{align*}
    as desired.
\end{proof}

\begin{restatable}{corollary}{privgcm}
\label{cor:privgcm}
    Let $a$ and $b$ be known bounds on the domains of $X$ and $Y$, respectively.  Given a dataset $\cD = (\bx, \by, \bz)$, let $(\hat{\bx}, \hat{\by}, \bz)$ be the rescaled dataset obtained by setting $\hat{\bx} = \bx/ a$ and $\hat{\by} = \by /b$. Let PrivGCM be the algorithm which runs \Alg{gcm} with kernel ridge regression as the fitting procedure $\mathcal{F}$ and sensitivity bound $\Delta = \gcmconst$, where $\gcmconst$ is the constant from \Lem{sensitivity_residuals}. Algorithm PrivGCM is $\eps$-differentially private.
    
    The statistic $T^{(n)} = T(R_1 + W_1, \dots, R_n + W_n)$, defined in \Alg{gcm}, satisfies the following.
    \begin{enumerate}
        \item \label{item:privgcm_level} Let $\cP \subset \cP_0$ be a family of distributions such that $A_f A_g = o_{\cP}(n^{-1}), B_f =  o_{\cP}(1), B_g = o_{\cP}(1)$. If in addition $\sup_{P \in \cP} \E[|\chi_P \xi_P|^{2+\eta}] \leq c$,  for some constants $c, \eta > 0$, then
        \begin{align*}
            \lim_{n \to \infty} \sup_{P' \in \cP'} \sup_{t \in \R} |\Pr_{P'}[T^{(n)} \leq t] - \Phi(t)| = 0.
        \end{align*}

    \item \label{item:privgcm_power} Define $\rho_P = \E_P[\cov(X, Y \mid Z)]$ and $\sigma_P = \sqrt{\Var_P(\chi_P\xi_P)}$.
    Let $\cP \subset \mathcal{E}_0$ be a family of distributions such that $A_f A_g = o_{\cP}(n^{-1}), B_f =  o_{\cP}(1)$ and $B_g = o_{\cP}(1)$. If in addition we have $\sup_{P \in \cP} \E[|\chi_P \xi_P|^{2+\eta}] \leq c$, for some constants $c, \eta > 0$, then
      \begin{align*} 
            \lim_{n \to \infty} \sup_{P' \in \cP'}  \sup_{t \in \R} \Bigl| \Pr_{P'} \Bigl[ T^{(n)} - \frac{\sqrt{n}\rho_P}{\sigma_P'} \leq t\Bigr] - \Phi(t) \Bigr| = 0,  
        \end{align*}
        where $\sigma_P' = \sqrt{ \sigma_P^2 +  (\frac{\sqrt{2}ab\cdot \gcmconst}{\eps})^2} $. 
    \end{enumerate}
\end{restatable}
\begin{proof}
    The fact that PrivGCM is $\eps$-differentially private follows from \Lems{laplace}{sensitivity_residuals}. Note that if the regularization parameter $\lambda$ is chosen adaptively based on the data, then obtaining the value of the constant $\gcmconst$ by plugging in $\lambda$ might not be $\eps$-differentially private. This can be resolved by setting an a priori lower bound on $\lambda$, independent of the data, e.g.,~$\lambda \geq 1$, and plugging that lower bound to obtain $\gcmconst$. 

    Item~\ref{item:privgcm_level} follows from \Thm{level_rescaled} and \Lem{sensitivity_residuals}. Item~\ref{item:privgcm_power} follows from \Thm{power_rescaled} and \Lem{sensitivity_residuals}
\end{proof}

\subsection{Auxiliary Lemmas}

\begin{lemma}[Slutsky's Lemma]\label{lem:slutksy} Let $A_n, B_n$ be sequences of random variables. If $A_n$ converges in distribution to $A$ and $B_n$ converges in probability to a constant $c$, then (1) $A_n + B_n \to A + c$, (2) $A_n \cdot B_n \to c \cdot A$, and (3) $A_n / B_n \to A/c$. 
\end{lemma}

\begin{lemma}[Uniform version of the Central Limit Theorem (Lemma 18 of \citet{ShahP20})] \label{lem:clt_uniform}
Let $\cP$ be a family of distributions for a random variable $\zeta$ such that for all $P \in \cP$ it holds $\E_P[\zeta] = 0$, $\E_P[\zeta^2] = 1$, and $\E_P[|\zeta|^{2+\eta}] < c$ for some $\eta, c > 0$. Let $(\zeta_i)_{i \in \N}$ be i.i.d copies of $\zeta$. For $n \in \N$, define $S_n = \frac{1}{\sqrt{n}} \sum_{i=1}^n \zeta_i$. Then
\begin{align*}
    \lim_{n \to \infty} \sup_{P \in \cP} \sup_{t \in \R} |\Pr_{P}[S_n \leq t] - \Phi(t)| = 0. 
\end{align*}
\end{lemma}

\begin{lemma}[Uniform version of the Weak Law of Large Numbers (Lemma 19 of \citet{ShahP20})] \label{lem:wlln_uniform}
    Let $\cP$ be a family of distributions for a random variable $\zeta$ such that for all $P \in \cP$ it holds $\E_P[\zeta] = 0$ and $\E_P[|\zeta|^{1+\eta}] < c$ for some $\eta, c > 0$. Let $(\zeta_i)_{i \in \N}$ be i.i.d copies of $\zeta$. For $n \in \N$, define $S_n = \frac{1}{n} \sum_{i=1}^n \zeta_i$. Then for all $\delta > 0 $ it holds
    \begin{align*}
       \lim_{n \to \infty} \sup_{P \in \cP} \Pr_{P}[|S_n| > \delta] = 0. 
    \end{align*}
\end{lemma}

\begin{lemma}[Uniform version of Slutsky's Lemma (Lemma 20 of \citet{ShahP20})]\label{lem:slutsky_uniform}
    Let $\cP$ be a family of distributions that determines the law of sequences $(A_n)_{n \in \N}$ and $(B_n)_{n \in N}$ of random variables. 
    \begin{enumerate}
        \item If $A_n$ converges uniformly over $\cP$ to $\cN(0, 1)$, and $B_n = o_{\cP}(1)$, then $A_n + B_n$ converges uniformly over $\cP$ to $\cN(0, 1)$. 
        \item If $A_n$ converges uniformly over $\cP$ to $\cN(0, 1)$, and $B_n = 1 + o_{\cP}(1)$, then $A_n / B_n$  converges uniformly over $\cP$ to $\cN(0, 1)$.
        \item If $A_n = M + o_{\cP}(1)$ for some $M > 0$ and $B_n = o_{\cP}(1)$, then $A_n + B_n = M + o_{\cP}(1)$. 
        \item If $A_n = o_{\cP}(1)$ and $B_n = o_{\cP}(1)$, then $A_n + B_n = o_{\cP}(1)$. 
    \end{enumerate}
\end{lemma}

\begin{claim}\label{clm:max_laplace}
    Let $(W_i)_{i \in \N}$ be i.i.d copies of $W \sim \Lap(0, \sigma)$. Let $S_n = \max_{i \in [n]}|W_i|$. Then $\E[S_n] \leq \sigma (1+\ln n)$. 
\end{claim}
\begin{proof}
    If $W_i \sim \Lap(0, \sigma)$, then $|W_i| \sim \mathrm{Exponential}(\sigma^{-1})$.  It is a known fact that $\E[\max{|W_1|, \dots, |W_n|}] = \sigma \cdot H(n)$, where $H(n)$ is the harmonic number $H(n) = \sum_{i=1}^n 1/i$. The claim follows from the fact that $H(n) \leq \ln(n) + 1$. 
\end{proof}

\begin{claim}\label{clm:moment_of_sum}
    Let $A$ and $B$ be random variables and $p > 1$. Then $\E[|A + B|^p] \leq 2^{p-1}(\E[|A|^p] + \E[|B|^p])$.
\end{claim}

\begin{claim}\label{clm:laplace_moment}
    Let $W\sim \mathrm{Laplace}(0, \sigma)$. Then for $n \in \mathrm{N}$ we have $\E[|W|^n] = \frac{n!}{\sigma^n}$. 
\end{claim}
\begin{proof}
    This claim follows from the fact that $|W| \sim \mathrm{Exponential}(\sigma^{-1})$. 
\end{proof}

\section{Proofs of \Sec{crt}} \label{app:crt}

In this section, we collect all missing proofs from \Sec{crt}.

\sen*
\begin{proof}
    Fix $c \in [0, m]$. First, we bound the sensitivity of $Q_c$. Suppose by contradiction that $|Q_c - Q_c'| > \Delta_T$.  Consider the case when $Q_c > Q_c'$. Then $Q_c > Q_c' + \Delta_T$. 
    Let $\mathtt{Above}_c = \{i \in [0,m] \mid T_i \geq Q_c\}$. Define $\mathtt{Above}'_c$ similarly. Then, for all $i \in  \mathtt{Above}_c$, we have
    \begin{align}
            T_i' \geq T_i - \Delta_T \geq Q_c - \Delta_T > Q_c', \label{eq:strictly_greater}
    \end{align}
    where the first inequality holds since the values $T_i$ have sensitivity at most $\Delta_T$, the second inequality holds from $i \in \mathtt{Above}_c$, and the last inequality holds from our assumption by contradiction.  
    Thus $i \in \mathtt{Above}'_c$, and as a result $\mathtt{Above}_c \subseteq \mathtt{Above}_c'$. Moreover, for $j$ such that $T_j' = Q_c'$ we have that $j \notin \mathtt{Above}_c$ since it does not satisfy \Eqn{strictly_greater}. Therefore $|\mathtt{Above}'_c| > |\mathtt{Above}_c|$, a contradiction. 

    For the case when $Q_c < Q_c'$ we obtain a contradiction by a symmetric argument. This concludes the proof on the sensitivity of $Q_c$. Next, we bound the sensitivity of the score function. We have,     
    \begin{align*}
        |s_k(c, \cD) - s_k(c, \cD)| \leq \frac{\Big| |Q_c - T_k| - |Q_c' - T_k'|\Big|}{2\Delta_T} \leq \frac{|Q_c - Q_c'| + |T_k -T_k'|}{2\Delta_T}.
    \end{align*}
    We just showed that $|Q_c - Q_c'| \leq \Delta_T$. Since the queries $T_i$ have sensitivity at most $\Delta_T$, we also have $|T_k - T_k'| \leq \Delta_T$. We obtain that $|s(c, \cD) - s(c, \cD)| \leq 1$ for all neighboring datasets $\cD, \cD'$. 
\end{proof}

\sensitivitycrt*
 \begin{proof}
    Suppose that $\cD$ and $\cD'$ differ in the last row. In the following, we assume that $j \in [m]$ is fixed. To ease notation, we remove the superscript $(j)$ from all $r_{X, i}^{(j)}, r_{Y, i}^{(j)}$ and $R_i^{(j)}$. Since we know $\E[X \mid Z]$ exactly, we have $r_{X, i}= x_i- \E[X \mid Z = z_i]$ for $i \in [n], j\in [m]$. Then for $i \in [n]$  we have
    \begin{align*}
        |r_{X, i}- r_{X, i}'| = |(x_i- \E[X \mid Z = z_i]) - (x_i' - \E[X \mid Z = z_i'])|.
    \end{align*}
    If $i \in [n-1]$, then $x_i= x_i'$ and $z_i = z_i'$, so that $|r_{X, i}- r_{X, i}| = 0$. For $i = n$, we have $|r_{X, n}- r_{X, n}| \leq 2$ by the triangle inequality and since $|r_{X, n}|\leq 1$. 

    Let $c_1 = 2 + \frac{2\sqrt{2}}{\sqrt{\lambda}}$ and $c_2 = \frac{8\sqrt{2}}{\lambda^{3/2}} + \frac{8}{\lambda}$. Turning to the residuals of fitting $\by$ to $\bz$, by the same argument as in \Lem{sensitivity_residuals} we have, for all $i\in[n-1]$, 
    \begin{align*}
        |r_{Y,i}- r_{Y,i}'| \leq \frac{8\sqrt{2}}{\lambda^{3/2}n} + \frac{8}{\lambda n} = \frac{c_2}{n}.  
    \end{align*}
    For the last datapoint it holds
      \begin{align*}
        |r_{Y,n}- r_{Y,n}'| \leq 2 + \frac{2\sqrt{2}}{\sqrt{\lambda}} + \frac{8\sqrt{2}}{\lambda^{3/2}n} + \frac{8}{\lambda n} = c_1 + \frac{c_2}{n}. 
    \end{align*}
    Additionally, $|r_{Y, i}| \leq c_1/2$ for all $i \in [n]$. 

    We can now bound the sensitivity of the residual products $R_i= r_{X, i}r_{Y, i}$. For $i \in [n-1]$, we have 
    \begin{align*}
        |R_i- R_i'| \leq |r_{X, i}|||r_{Y, i} - r_{Y, i}'| + |r_{Y, i}'||r_{X, i} - r_{X,i}'| \leq \frac{c_2}{n} + 0. 
    \end{align*}
    For $i = n$ we have
    \begin{align*}
        |R_n - R_n'| \leq \Big(c_1+\frac{c_2}{n}\Big) + \frac{c_1}{2}\cdot 2 = 2c_1 + \frac{c_2}{n}. 
    \end{align*}
    Finally,
    \begin{align*}
         |T_j - T_j'| = \Big|\sum_{i=1}^n R_i -  \sum_{i=1}^n R_i'\Big|\leq \sum_{i=1}^n |R_i - R_i'| \leq (n-1) \cdot \frac{c_2}{n} + 2c_1 + \frac{c_2}{n} = 2c_1 + c_2.
    \end{align*}
\end{proof}

\crt*
\begin{proof}
We first show that PrivCRT is $\eps$-differentially private. The scores $T_i, i\in [0,m]$ have sensitivity at most $\Delta_T = O(1)$ by \Lem{sensitivity_crt}. Therefore, the scores $s_i, i \in [0, m]$ have sensitivity at most $1$ by \Lem{sensitivity_score}. Finally, by \Thm{rnm} we have that outputting $\hat{c}$ is $\eps$-DP. Therefore, PrivCRT is $\eps$-DP. Note that if the regularization parameter $\lambda$ is chosen adaptively based on the data, then obtaining the value of the constant $\crtconst$ by plugging in $\lambda$ might not be $\eps$-differentially private. This can be resolved by setting an a priori lower bound on $\lambda$, independent of the data, e.g.,~$\lambda \geq 1$, and plugging that lower bound to obtain $\crtconst$. 

Next, we analyze the accuracy of PrivCRT. Let $c^*$ be the true rank of $T_0$ amongst the statistics $\{T_i\}_{i \in [0, m]}$, sorted in decreasing order. Then $p^* = (1+c^*)/(m+1)$. Note that $s_{c^*} = 0$. By \Thm{rnm}, with probability at least $1-\delta$, it holds 
\begin{align*}
    s_{\hat{c}} \geq - \frac{2\log(m/\delta)}{\eps}.
\end{align*}
As a result, 
\begin{align*}
    |Q_{\hat{c}} - T_0| = -2\Delta_T \cdot s_{\hat{c}} \leq \frac{4\Delta_T \log(m/\delta)}{\eps}.
\end{align*}
Let $\gamma = \frac{4\Delta_T \log(m/\delta)}{\eps}$. The rank $c^*$ of $T_0$ cannot differ from $\hat{c}$ by more than $G_\gamma$, since $Q_{\hat{c}}$ is within distance $\gamma$ of $T_0$. Therefore, $|\hat{p} - p^*| \leq \frac{G_\gamma}{m+1}$, and since $\Delta_T = O(1)$, we  obtain the desired result. 
\end{proof}

\section{Additional Experimental Details and Results}\label{app:extra_experiments}
\paragraph{Example Dataset.} In \Fig{data} we show one example dataset from our simulations. We fit a kernel ridge regression model of $X$ to $Z$. As can be observed, the model we fit  matches $\E[X|Z]$ closely.

\begin{figure}[!h]
\centering
\includegraphics[width = 0.6\textwidth]{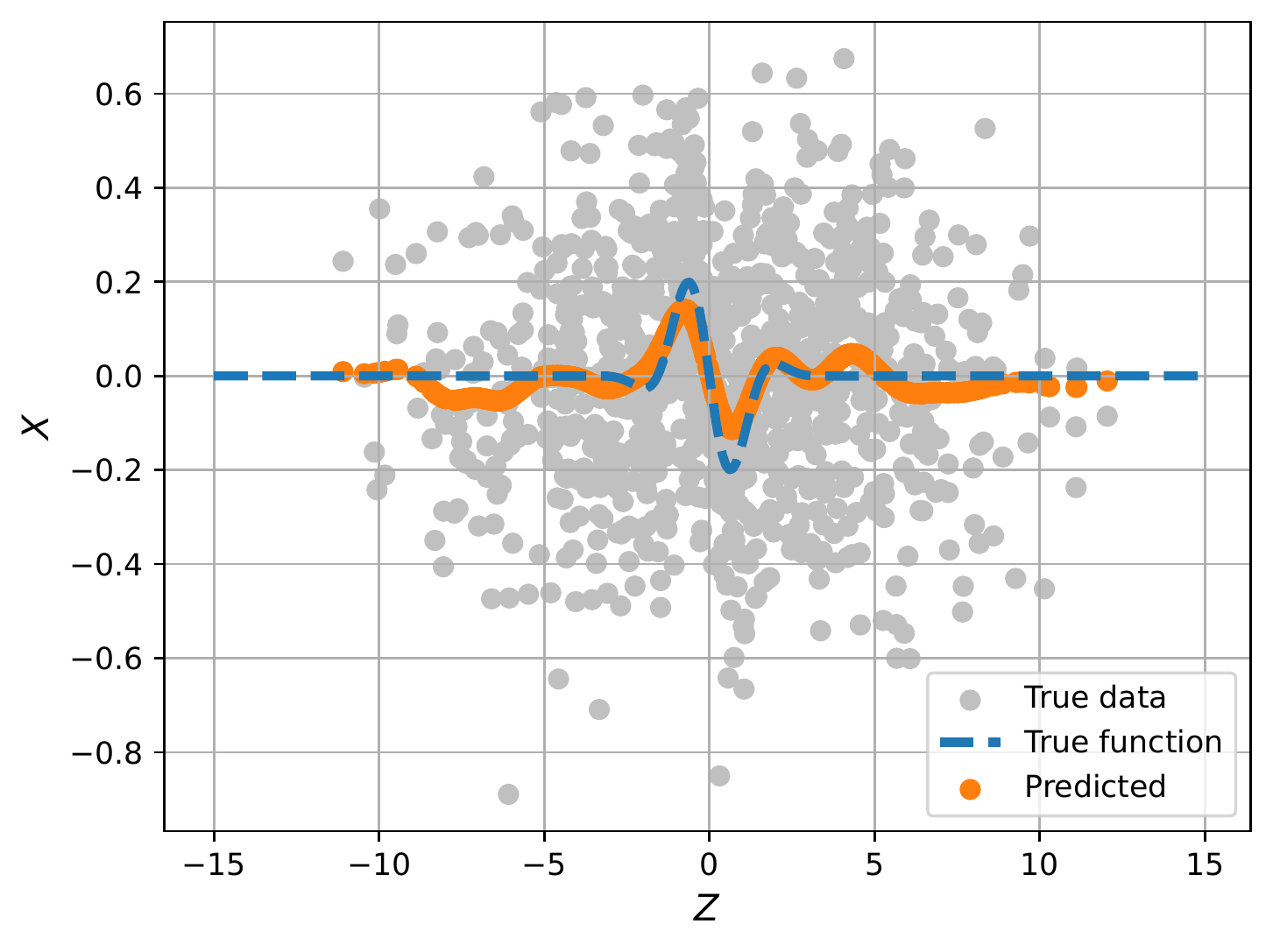}
\caption{Values of $X$ and $Z$ (after rescaling) of one sampled dataset from our simulations, with $n = 1000$, $\beta = 0$, $s=2$, $d=1$. A kernel ridge regression model is fitted to the data. The model we fit closely matches $\E[X|Z]$.}
\label{fig:data}
\end{figure}

\paragraph{Implementation of Private Kendall and PrivToT.} The implementation of PrivToT requires setting a parameter $k$ for the number of subsets into which the original dataset is divided. We use $k=10$ for the experiments in \Fig{kendall}. We additionally experimented with $k \in [20,50,100]$ and found that $k=10$ performs the best in terms of type-I error. The implementation of PrivToT was adapted from the implementation of~\citet{KazanSGB23}. The private Kendall test works for categorical $Z$. To adapt it to continuous $Z$ we apply $k$-means clustering to $Z$ to obtain $100$ clusters. The implementation of private Kendall was adapted from~\citet{WangKLK18}. 


\begin{figure}[!h]  
\begin{center}
\centering
\includegraphics[width = 0.8\textwidth]{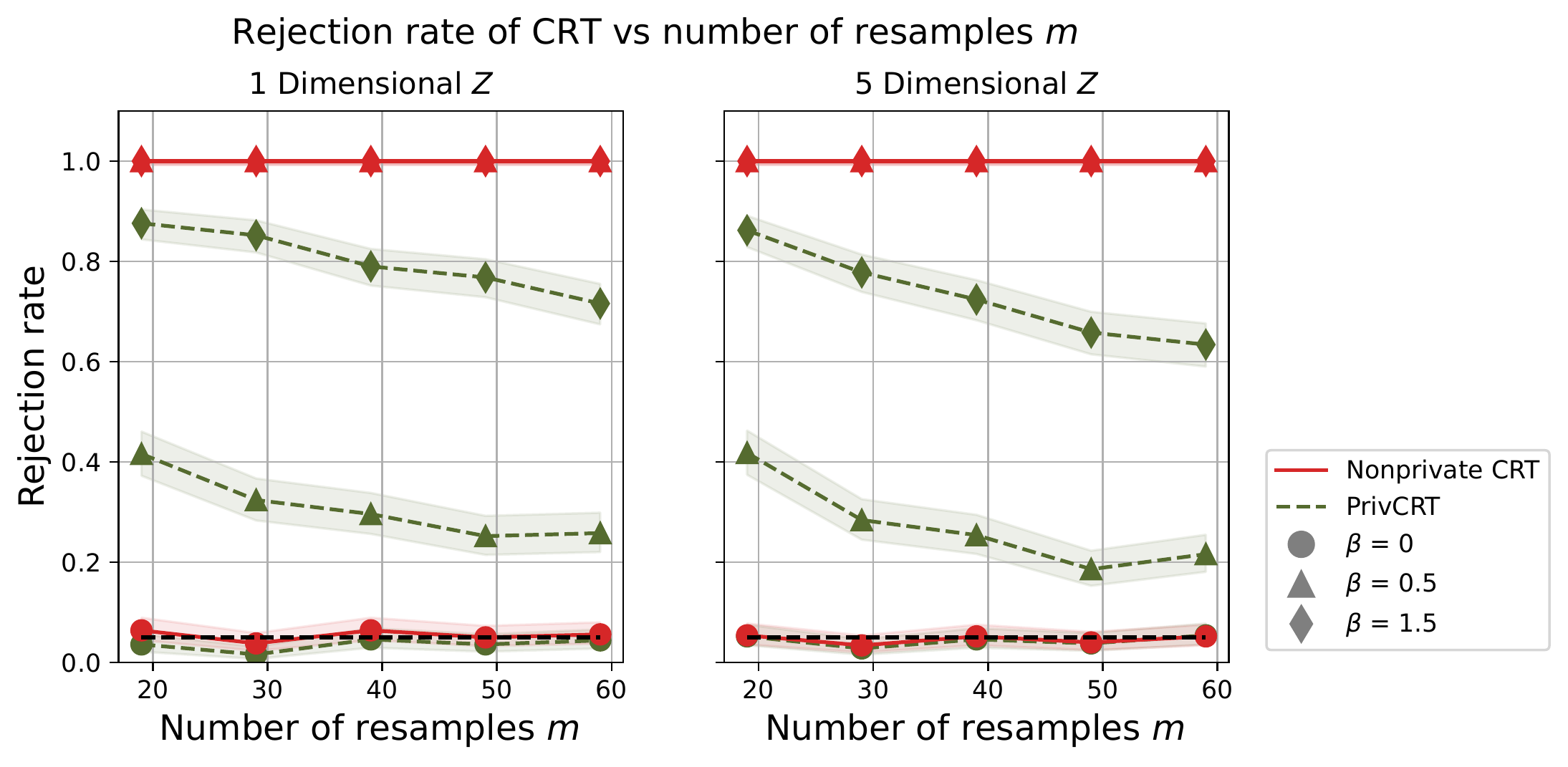}
\caption{Effect on the power of PrivCRT with increasing $m$.}
\label{fig:crt_m}
\end{center}
\end{figure}

\begin{figure}[!h]
\centering
\includegraphics[width = \textwidth]{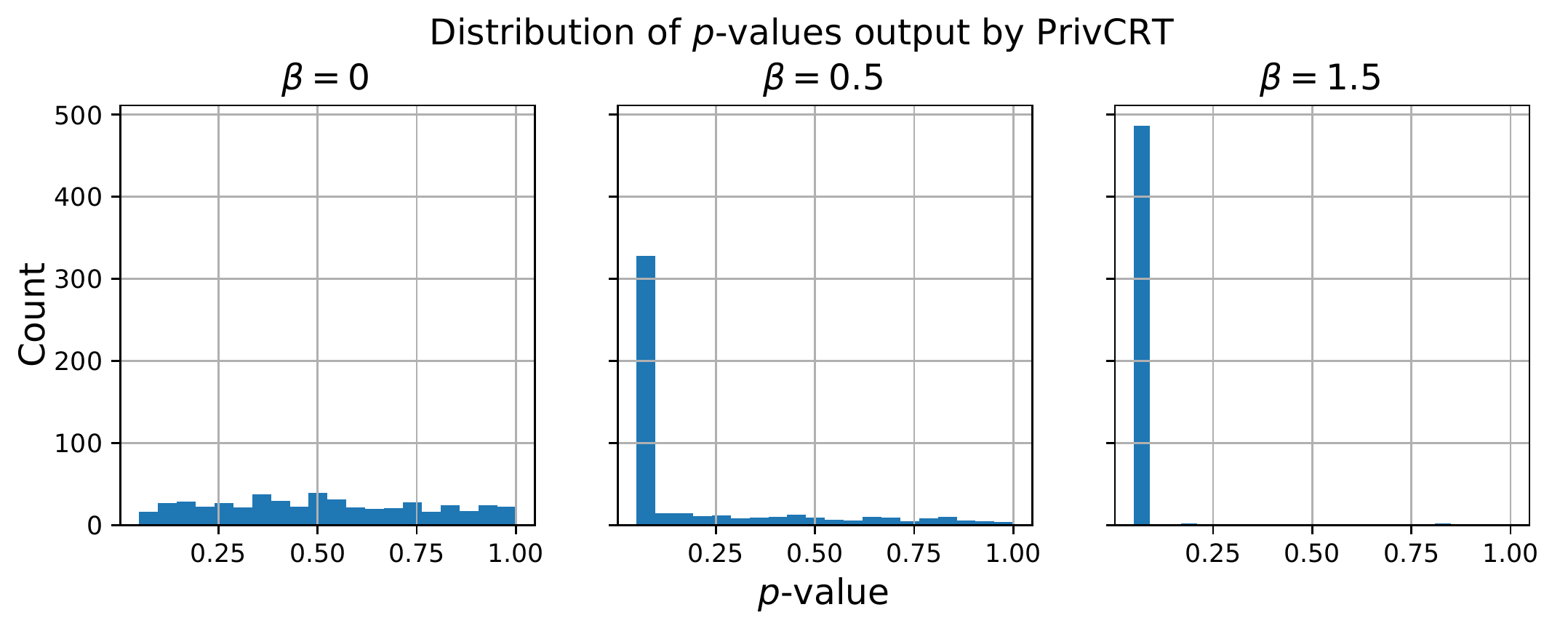}
\caption{Distribution of p-values output by PrivCRT for different dependence strengths $\beta$ under the setup in Section~\ref{sec:experiments}. Under the null, i.e., $\beta=0$, the p-values are uniformly distributed as desired.}
\label{fig:pvals}
\end{figure}


\paragraph{The Effect of Varying $m$ on the Power of PrivCRT. } 
In \Fig{crt_m}, we vary $m$, the number of resamples used in the CRT algorithm and run our experiments for $\beta \in \{0, 0.5, 1.5\}$. We set $n = 10^3$ and $\eps = 2$. Increasing $m$ does not affect the type-I error control of PrivCRT.  However, we observe that the power of PrivCRT decreases as the number of resamples $m$ increases. This is due to the increase of the variable $G_{\gamma}$ with $m$ (see \Def{c_gamma}). More specifically, for a fixed $\gamma$, as $m$ increases, $T_0$ stays fixed, while the number of other statistics $T_i$ within distance $\gamma$ of $T_0$ increases. Thus, the private test is more likely to select a rank that is further from the true rank. It is an interesting open question whether the dependence on $m$ in the accuracy of a private CRT test is avoidable. For now, we recommend using $m = O(1/\alpha)$ when employing PrivCRT.

\paragraph{Distribution of p-values Output by PrivCRT. } In \Fig{pvals} we show the distribution of the p-values output by PrivCRT for different dependence strengths $\beta$. We set $n = 10^3$, $m=19$, and $\eps=2$. Under the null, i.e., when $\beta = 0$, the p-values output by PrivCRT are uniformly distributed in the interval $[\frac{1}{m+1}, 1]$. Thus, PrivCRT controls type-I error. When $\beta > 0$, most of the p-values are close to $\alpha = 0.05$, which is the desired outcome for PrivCRT to achieve power.

\end{document}